\documentclass{article}

\usepackage{natbib}
\usepackage[final]{neurips_2021}

\usepackage[utf8]{inputenc} 
\usepackage[T1]{fontenc}    
\usepackage{hyperref}       
\usepackage{url}            
\usepackage{booktabs}       
\usepackage{amsfonts}       
\usepackage{nicefrac}       
\usepackage{microtype}      
\usepackage{xcolor}         

\usepackage{microtype}
\usepackage{graphicx}
\usepackage{subfigure}
\usepackage{booktabs} 
\usepackage{hyperref}

\usepackage{amsthm} 
\usepackage{amsmath}
\usepackage{amssymb}
\usepackage[user,titleref]{zref}
\usepackage{dsfont}

\title{Infinite-channel deep convolutional Stable neural networks}

\author{%
   Daniele Bracale\\
   Department of Statistics\\
   University of Michigan\\
   \texttt{dbracale@umich.edu} \\
   \And
   Stefano Favaro \\
   Department ESOMAS\\
   University of Torino and \\
   Collegio Carlo Alberto \\
   \texttt{stefano.favaro@unito.it} \\
   \AND
   Sandra Fortini \\
   Department of Decision Sciences \\
   Bocconi University \\
   \texttt{sandra.fortini@unibocconi.it} \\
   \And
   Stefano Peluchetti \\
   Cogent Labs \\
   \texttt{speluchetti@cogent.co.jp} \\
}

\newcommand{\E}{\mathop{\mathbb{E}}}

\newcommand{\R}{\mathbb{R}}
\newcommand{\N}{\mathbb{N}}
\newcommand{\Z}{\mathbb{Z}}

\newcommand{\patch}{\scalebox{0.5}{$\bigstar$}}
\DeclareMathOperator{\1}{\mathds{1}}
\DeclareMathOperator{\ii}{\mathrm{i}}
\DeclareMathOperator{\dd}{\mathrm{d}}

\newtheorem{theorem}{Theorem}
\newtheorem{prop}{Proposition}
\newtheorem{lemma}{Lemma}
\newtheorem{definition}{Definition}

\newenvironment{sistema}%
  {\left\lbrace\begin{array}{@{}l@{}}}%
  {\end{array}\right.}

\begin{document}

\maketitle

\begin{abstract}
  The connection between infinite-width neural networks (NNs) and Gaussian processes (GPs) is well known since the seminal work of \citet{neal2012bayesian}. While numerous theoretical refinements have been proposed in recent years, the connection between NNs and GPs relies on two critical distributional assumptions on the NN's parameters: i) finite variance ii) independent and identical distribution (iid). In this paper, we consider the problem of removing assumption i) in the context of deep feed-forward convolutional NNs. We show that the infinite-channel limit of a deep feed-forward convolutional NNs, under suitable scaling, is a stochastic process with multivariate stable finite-dimensional distributions, and we give an explicit recursion over the layers for their parameters. Our contribution extends recent results of \citet{stablebehaviour} to convolutional architectures, and it paves the way to exciting lines of research that rely on GP limits.
\end{abstract}

\section{Introduction}
Fully-connected NNs are defined by an interleaved application of affine transforms and non-linear functions evaluated element-wise. Associating a distribution to the parameters of a NN allows us to consider the NN as a probabilistic model. Modern NNs typically operate in the over-parametrized regime, with millions of parameters representing a standard setting. 
As a way forward, \citet{neal2012bayesian} established the equivalence between a certain class of shallow probabilistic NNs and corresponding limiting GPs when the NN's width, hence the dimensionality of its parameters, becomes infinite. 
Recently, the NN-GP correspondences has been extended to deep fully-connected \citet{lee2017deep, matthews2018gaussianB} and convolutional \citet{novak2018bayesian,garriga2018deep} NNs. Two common properties of the NN's parameters' distribution underlay the NN-GP correspondence: i) finite variance ii) iid distribution. In this paper we extend the results of \citet{novak2018bayesian,garriga2018deep} to iid parameters distributed according to a Stable distribution (SD), effectively removing assumption i). More precisely, we study the infinite-channel limit of CNNs in the following general setting: i) the CNN is deep, namely is composed of multiple layers; ii) biases and scaled weights are iid according to a centered symmetric SD; iii) number of convolutional channels in each network’s layers goes to infinity jointly on the layers; iv) the convergence in distribution is established jointly for multiple inputs, namely the convergence concerns the class of finite dimensional distributions of the CNN viewed as a stochastic process. The use of SDs, which includes the Gaussian distribution as special case, is natural in this setting. Indeed, SDs are the most general distribution class toward which infinite sums of iid random variables can converge in law \citet{samoradnitsky2017stable}. Through this paper we define a deep CNN (DCNN) of the form $f^{(1)} = W^{(1)}* x + b^{(1)}$ and $f^{(l)} = W^{(L)} * \phi (f^{(l-1)}) + b^{(l)}$ for $l=2, \dots, L$, where $l$ indexes the $L$ layers, $\{W^{(l)}\}_{l=1}^L$ are the weights, $\{b^{(l)}\}_{l=1}^L$ are the biases and $\phi$ is an activation function. We show that the infinite-channel limit of the DCNN, under suitable scaling on the weights, is a stochastic process whose finite-dimensional distributions are multivariate SDs \citet{samoradnitsky2017stable}. This process is referred to as the convolutional stable process (CSP). Our result contributes to the theory of DCNNs, and it paves the way to extend the research directions that rely on Gaussian infinite wide limits.

The paper is structured as follows. Section \ref{sec:definitions} introduces the notation and definitions. In Section \ref{sec:DCNN} we define a DCNN jointly on $K$ distinct inputs and we specify the distribution assumptions on the model parameters. In Section \ref{sec:limits} we compute the limiting distributions jointly over $K$ inputs and we establish the distribution, again over $K$ inputs, arising from the projection of all spatial features to a single output vector (i.e. a readout layer). In Section \ref{sec:conclusions} we conclude.

\section{Notation and definitions}\label{sec:definitions}
Denote: $[n]$ the set $\{1,\dots, n\}$ $ \forall n \in \N$; fixed a size $S \in \N$, an $S$-tensor of dimension $\textbf{D}_S := D_1 \times \dots \times D_S$ (where $D_j \in \N$ $\forall j \in [S]$), is an element $A \in \R^{\textbf{D}_S}$; $[\textbf{D}_S]:= [D_1] \times \dots \times [D_S]$; $|\textbf{D}_S|:= \prod_{s=1}^{S} D_s$; $A_{d} \in \R$ (where $d =(d_1,\dots,d_S) \in [\textbf{D}_S]$) is the component of $A$ at position $d$; the norm $\|A\|^{\alpha}=\sum_{d \in [\textbf{D}_S]} A_{d}^{\alpha}$; $A_{(d_1,;)}$ is the $(S-1)$-tensor of dimension $1 \times D_2 \times \dots \times D_{S}$ consisting on the $d_1$-th position of the first dimension of $A$ and all the other positions of $A$; an integral in $\dd(A_{\{d \in [\textbf{D}_S]\}})$ is an integral in the flattened tensor $A$, i.e. with $|\textbf{D}_S|$ variables of integration; $\1 _{(\textbf{D}_{S})}$ is the tensor of all $1$s in $\R^{\textbf{D}_{S}}$ and $\1_{(\textbf{D}_{S})}[d]$ 
the one with $1$ in the $d$-th entry and zero otherwise; let $A,B \in \R^{\textbf{E}_S}$, $A \otimes B = \sum_{e \in [\textbf{E}]} A_{e} B_{e} \in \R$ is the \textit{Frobenius product}; let $(A,B) \in  \R^{\textbf{D}_S \times \textbf{E}_{S'}} \times \R^{\textbf{E}_{S'}}$, $A\boxdot_{\diamondsuit} B \in \R^{\textbf{D}_{S}}$ is called \textit{square product under $\diamondsuit$}(where $\diamondsuit$ operates within tensors of the same size into $\R$, e.g. the Frobenius product), where each position $d \in[\textbf{D}_{S}]$ of $A \boxdot_{\diamondsuit} B$ is $A_{(d,:)} \diamondsuit B \in \R$; $\boxdot$ is simply $\boxdot_{\otimes}$ and called \textit{square product}; let $(A,B) \in  \R^{\textbf{D}_S} \times \R^{\textbf{E}_{S'}}$, $A \triangle B = (AB_{b})_{\{b \in [\textbf{E}_{S'}]\}} \in \R^{\textbf{D}_{S} \times \textbf{E}_{S'}}$ is called \textit{bias product}. When we specify some dimensions over an operation, it means that the operation is applied through \textit{all dimensions except for} the specified ones, e.g. in $\overset{(P,K)}{\boxdot}$, $\boxdot$ is applied through all dimensions except for dimensions $P$ and $K$. When dimensions are under the operation, the operation is applied \textit{only to the specified dimensions}; $A \sim \text{St}_{\textbf{D}_S}(\alpha,\Gamma)$ indicates that the corresponding flattened vector has characteristic function $\varphi_A(t):= \E[e^{\ii \textbf{t} \otimes A }]=\exp\{-\int_{\mathbb{S}^{|\textbf{D}_S|-1}}| \textbf{t} \otimes \textbf{s} |^{\alpha}\Gamma(\dd \textbf{s})\}$ for all $t \in \R^{\textbf{D}_S}$, where $\Gamma$ is a spectral measure on $\mathbb{S}^{|\textbf{D}_S|-1}=\{z\in \R^{|\textbf{D}_S|}: \|z\|=1\}$ (more details on ``\ztitleref{sec:app:def}").

\section{Stable convolutional networks}\label{sec:DCNN}
\textbf{Shallow CNN.} The input to a convolution is a tensor $x \in \R^{C \times \textbf{P}_S}$ where $S$ is the number of spacial dimensions, $P_s$ is the size of the $s$-th spacial dimension $\forall s \in [S]$ and $C$ is the number of channels. The defining property of a convolution is that the same collection of filters (weights) is applied to multiple patches extracted from the input tensor $x$. The filter size must thus agree with the extracted patches sizes. There is great flexibility in defining the specific details of a given convolutional transform, including its striding, padding, and dilation characteristics. See \citet{dumoulin2016guide} for a comprehensive account. In this paper we consider the following general setting. We define the filters $W \in \R^{C' \times C \times \textbf{G}_S}$ where $G_s$ ($\leq P_s$ and the equality makes the network a $S$-dimensional fully connected NN) is the filter size across the $s$-th space dimension $\forall s \in [S]$ and $C'$ is the number of output channels. We also define a bias term $b \in \R^{C'}$. A convolution transform over $x$ results in an output tensor $y \in \R^{C' \times \textbf{P}'_S}$ where $P'_s$ is the size of the $s$-th spacial output dimension. $P'_s$ depends on both $P_s$ and on the convolution characteristics. We write $y=W*x+b$, where the product $*$ is defined as follows. A patch extracted from $x$ at output position $p$ is $x_{\patch p} \in \R^{C \times \textbf{G}_S}$, $x_{\patch p} = x_{1:C,\patch p}$ where $\patch: p \mapsto \patch p$ is a function that depends on the moving window chosen in the structure and returns the positions of $x$ associated with the corresponding output position $p$. When extraction happens outside of $x$, i.e. when an input position $i$ is such that $i < 1$ or $i > P$, the padded values (often a constant) are taken as input. The convolution transform at output position $p$ is thus given by $y_p = W x_{\patch p} + b \in \R^{C'}$. Finally $y \in \R^{C'}$ is obtained by stacking $y_p$ over the $\textbf{P}'_S$ output positions, i.e. the NN can be rewritten as
\begin{equation}\label{singlelayer2}
    y = [W x_{\patch p} + b]_{\{p \}}
\end{equation}
\textbf{DCNN with K inputs.} We extend the definition (\ref{singlelayer2}). A DCNN is defined by multiple layers of convolutional transforms followed by the application of an element-wise activation function $\phi$. We consider the case where all layers have the same number of channels $C^{(1)}= \dots = C^{(L)}=C$. Define $\textbf{P}^{(l)}=\textbf{P}^{(l)}_{S^{(l)}}$ for $l \in [L]\cup \{0\}$ and $\textbf{G}^{(l)}= \textbf{G}^{(l)}_{S^{(l-1)}}$ for $l \in [L]$. Consider $K$ inputs $x^{(k)} \in \R^{ C^{(0)}\times \textbf{P}^{(0)}}$, $k \in [K]$ and define $x^{(1:K)}=(x^{(1)}, \dots, x^{(K)})^T \in \R^{C^{(0)} \times \textbf{P}^{(0)} \times K}$. For each $x^{(k)}$ the convolutional structure remains constant, that is the filters and the bias terms do not depend on $k$. Since they do not even depend on the positions, a DCNN of $L$ layers with with $K$ inputs can be defined as follows
\begin{equation}\label{eq:conv_expl_joint}
\begin{cases}
f^{(0)(1:K)} = x^{(1:K)} & \in \R^{C^{(0)} \times \textbf{P}^{(0)} \times K} \\
f^{(1)(1:K)} = f^{(l)}(x^{(1:K)}) = W^{(1)} \overset{(\textbf{P}^{(1)},K)}{\boxdot} x^{(1:K)}_{\patch}+b^{(1)} \triangle \1 _{( \textbf{P}^{(1)} \times K)} & \in \R^{\infty \times \textbf{P}^{(1)} \times K}\\
f^{(l)(1:K)} = f^{(l)}(x^{(1:K)},C)=\tfrac{1}{C^{1/\alpha}}W^{(l)} \overset{ (\textbf{P}^{(l)},K) }{\boxdot}\phi(f^{(l-1)(1:K)}_{\patch})\\
\qquad \qquad \qquad +b^{(l)} \triangle \1 _{( \textbf{P}^{(l)}\times K)} & \in \R^{\infty \times \textbf{P}^{(l)} \times K}
\end{cases}
\end{equation}
where the last holds for $l=2, \dots, L$ and where $\alpha \in (0,2]$, $W^{(l)} \in \R^{\infty \times C^{(l-1)} \times \textbf{G}^{(l)}}$, $b^{(l)} \in \R^{\infty}$ for each $l \in [L]$, and $x^{(1:K)}_{\patch}= [x^{(k)}_{\patch p^{(1)}} ]_{\{(p^{(1)},k) \in  [\textbf{P}^{(1)} \times K]\}} \in \R^{C^{(0)} \times \textbf{G}^{(1)} \times \textbf{P}^{(1)} \times K}$ and $f^{(l-1)(1:K)}_{\patch}=[f^{(l-1)(k)}_{\patch p^{(l)}}]_{\{(p^{(l)},k) \in [\textbf{P}^{(l)}\times K]\}} \in \R^{C^{(l-1)} \times \textbf{G}^{(l)} \times \textbf{P}^{(l)}  \times K}$ for $l= 2,\dots L $, where $f^{(l-1)}_{\patch p^{(l)}}:= f^{(l-1)}_{(1:C^{(l-1)},\patch p^{(l)})} \in \R^{C^{(l-1)} \times \textbf{G}^{(l)}}$ is a patch extracted from $f^{(l-1)}$ at output position $p^{(l)} \in [\textbf{P}^{(l)}]$. We defer to ``\ztitleref{sec:app:patch}" for a more detailed explanation.

\textbf{Assumptions on the parameters and on the activation function.}
\begin{align}
& \textbf{H1) } \forall l \in [L], c^{(l)} \geq 1, g^{(l)} \in \textbf{G}^{(l)} \text{:}\quad W^{(l)}_{(c^{(l)},c^{(l-1)},g^{(l)})} \overset{\text{iid}}{\sim} \text{St}(\alpha,\sigma_w)\quad iid \quad  b^{(l)}_{c^{(l)}} \overset{\text{iid}}{\sim} \text{St}(\alpha,\sigma_b) \label{eq:param_dist2}
\\
& \textbf{H2) } \phi: \R \to \R \text{ with finite discontinuities: }\forall s \in \R \text{ } |\phi (s) | \leq a+b|s|^{\beta} \text{ some $a,b>0$ and $\beta<1$} \label{envelope}
\end{align}

\section{Main theorems: infinitely wide limits}\label{sec:limits}
We study the limiting distribution of $f^{(l)(1:K)} = f^{(l)}(x^{(1:K)},C)$ as $C \rightarrow \infty$, $\forall l \in [L]$. Let $f^{(l)(1:K)}_{\infty} \in \R^{\infty \times \textbf{P}^{(l)} \times K}$ be this limit (i.e. the joint limit random variable over all (infinite) channels, positions and inputs). By the Cram\'er-Wold theorem it is sufficient to prove the large $C$ asymptotic behavior of any linear combination of $f^{(l)(1:K)}_{(c^{(l)},:)}$'s (see, e.g. \citet{billingsley1999convergence} for details), where $f^{(l)(1:K)}_{(c^{(l)},:)}$ is the $c^{(l)}$-th channel of $f^{(l)(1:K)}$, and, from (\ref{eq:conv_expl_joint}), it can be rewritten as
\begin{equation}\label{sistem:fixedchannel}
\begin{sistema}
f^{(1)(1:K)}_{(c^{(1)},:)} =  W^{(1)}_{(c^{(1)},:,:)} \overset{(\textbf{P}^{(1)},K)}{\otimes} x^{(1:K)}_{\patch }+ b^{(1)}_{c^{(1)}} \1 _{( \textbf{P}^{(1)} \times K)} \quad \in \R^{\textbf{P}^{(1)}\times K} \\
f^{(l)(1:K)}_{(c^{(l)},:)} = \frac{1}{C^{1/\alpha}} W^{(l)}_{(c^{(l)},:,:)} \overset{(\textbf{P}^{(l)},K)}{\otimes}  \phi ( f^{(l-1)(1:K)}_{\patch}) +b^{(l)}_{c^{(l)}}  \1 _{( \textbf{P}^{(l)} \times K)} \quad \in \R^{\textbf{P}^{(l)}\times K}, l=2,\dots,L
\end{sistema}
\end{equation}
Note: fully connected NNs are special cases when $P^{(l)}=1$ $ \forall l \in [L] \cup \{0\}$ and the patch extraction corresponds to the whole input for each convolutional transform. Thus we prove: \textbf{Theorem \ref{teorem3}} by fixing $c^{(L)} \geq 1,l \in [L]$ and computing the limit distribution of $f^{(l)(1:K)}_{(c^{(l)},:)}$ as $C \rightarrow \infty$ and \textbf{Theorem \ref{thm_Cramer}} by applying the Cram\'er-Wold theorem. We denoted with $\stackrel{d}{\rightarrow}$, $\stackrel{p}{\rightarrow}$ and $\stackrel{a.s.}{\rightarrow}$ 
respectively the convergence in distribution, in probability and almost surely. Our proof is an alternative to the Strong Law of Large Numbers for Stable random variables. The key point of the proof lies in recognizing the exchangeability of the sequence $(f^{(l)(1:K)}_{(c^{(l)},:)} )_{c^{(l)}\geq 1}$ which allows us to apply the de Finetti theorem. First, define for each $l \in [L]$ the function $\Psi^{(l)}:\R^{\textbf{P}^{(l)}\times K} \to \R,
\Psi^{(l)} (z) := \tfrac{1}{2}\delta (z/\| z \|) +\tfrac{1}{2} \delta (-z/\| z \|)$ if $ z \neq 0$ and $0$ otherwise, where $\delta$ denotes the Dirac delta function. Recall that we don't compute the limit for the layer $l=1$, because $f^{(1)(1:K)}_{(c^{(1)},:)}$ is referred only to the $C^{(0)}$ channels of the input layer. 
\begin{theorem}\label{teorem3}[\ztitleref{sec:appA}]
For each $l = 2, \dots, L$,  $f^{(l)(1:K)}_{(c^{(l)},:)} \overset{d}{\rightarrow} f^{(l)(1:K)}_{\infty (c^{(l)},:)} \sim \text{St}_{\textbf{P}^{(l)} \times K}(\alpha,\Gamma^{(l)}_{\infty})$ as $C \rightarrow \infty$, being $\Gamma^{(l)}_{\infty}$ equal to \[\| \sigma_b \1_{(\textbf{P}^{(l)} \times K)} \|^{\alpha} \Psi^{(l)} \Big( \1_{(\textbf{P}^{(l)} \times K)} \Big)+ \int \sum_{g^{(l)} \in [\textbf{G}^{(l)}]} \|\sigma_{\omega} \phi (f_{g^{(l)}})\|^{\alpha}  \Psi^{(l)} \Big( \phi (f_{g^{(l)}}) \Big) q^{(l-1)}(\dd f_{\{g^{(l)}\in [\textbf{G}^{(l)}]\}})\] where $f_{g^{(l)}} \in \R^{\textbf{P}^{(l)} \times K}$ for each $g^{(l)} \in \textbf{G}^{(l)}$, $q^{(l-1)} = \text{St}_{\textbf{P}^{(l-1)} \times K}(\alpha,\Gamma^{(l-1)}_{\infty})$, and
\begin{align*}
    \Gamma^{(1)}_{\infty}=\Gamma^{(1)}=& \| \sigma_b \1_{(\textbf{P}^{(1)} \times K)} \|^{\alpha} \Psi^{(1)} ( \1_{(\textbf{P}^{(1)} \times K)} )\\
    &\quad + \sum_{(c^{(0)},g^{(1)}) \in [C^{(0)} \times \textbf{G}^{(1)}]} \|\sigma_{\omega} (x^{(1:K)}_{\patch})_{(c^{(0)},g^{(1)})}\|^{\alpha} \Psi^{(1)} \Big(  (x^{(1:K)}_{\patch})_{(c^{(0)},g^{(1)})}\Big)
\end{align*}
\end{theorem} 

\begin{theorem}\label{thm_Cramer}[\ztitleref{sec:appB}]
For each $l \in [L]$, $f^{(l)(1:K)} = f^{(l)}(x^{(1:K)},C) \overset{d}{\rightarrow} f^{(l)(1:K)}_{\infty} \sim \bigotimes_{c^{(l)}=1}^{\infty} \text{St}_{\textbf{P}^{(l)} \times K}(\alpha,\Gamma^{(l)}_{\infty}) $ as $C \rightarrow \infty$, where the symbol $\bigotimes$ denotes the product measure.
\end{theorem}

\textbf{Readout layer on positions.} Fix $l=L$. We found that $f^{(L)(1:K)} \overset{d}{\rightarrow} f^{(L)(1:K)}_{\infty}$, i.e. a convergence of a sequence of $\R^{\infty \times \textbf{P}^{(L)} \times K}$-valued random variables. To gather information on the positions $\textbf{P}^{(L)}$ we consider a linear combination with respect to $\textbf{P}^{(L)}$, i.e. we project the $\infty \times \textbf{P}^{(L)} \times K$ dimensional vector $f^{(L)(1:K)}=f^{(L)}(x^{(1:K)},C)$ into one $\infty \times K$ dimensional, and we take the limit as $C \rightarrow \infty$. For a detailed explanation we defer to the ``\ztitleref{sec:appC}".

\section{Conclusions}\label{sec:conclusions}
We showed that an infinite-channel DCNN with scaled Stable parameters defines a stochastic process whose finite-dimensional distributions are multivariate SDs. The finite-dimensional distributions can be evaluated via an explicit recursion over the layers of the NN. We also established the finite-dimensional distributions arising from the NN readout layer. Several interesting theoretical developments are possible. Firstly, the results provided so far constitute the first step in establishing an NTK limit \citet{arora2019exact} arising from Stable distributed parameters in convolutional architectures. Secondly, all the established convergence results are limited to the finite dimensional distributions of the NN layers. This is not enough to guarantee the convergence of the NN seen as a random function of the input space, i.e. to establish a functional limit. Further effort is thus needed to extend the results of \citet{bracale2021largewidth} to the Stable and convolutional setting for $0<\alpha <2$. Doing so will also provide estimates on the (reduced, compared to the Gaussian case) smoothness proprieties of the limiting stochastic processes. Finally, it is necessary to devise efficient inference algorithms that allows to apply the stochastic processes introduced in this paper to current computer vision problems.
\bibliography{main}

\appendix

\newpage
\section*{\centering SUPPLEMENTARY MATERIALS}

\section*{SM: Stable random variables}\zlabel{sec:app:def}

Fix $0<\alpha \leq 2$ and $\sigma >0$ and define the following distributions.
\begin{definition} A $\R$-valued random variable $A$ is distributed as a SD with index $\alpha$, skewness $\tau \in [-1,1]$, scale $\sigma$ and shift $\mu \in \R$, and we write $A \sim \text{St}(\alpha, \tau, \sigma, \mu)$, if the characteristic function is
\[
\varphi_A(t) = \E[e^{\ii tA}] = e^{ \psi(t)}, \quad t \in \R
\quad \text{ where }
\]

\[
\psi(t)=
\begin{cases}
-\sigma^{\alpha} |t|^{\alpha}[1+ \ii \tau \tan(\frac{\alpha \pi}{2})\text{sign}(t)] +\ii \mu t & \alpha \neq 1 \\
-\sigma |t|[1+ \ii \tau \frac{2}{\pi}\text{sign}(t) \log(|t|)]+\ii \mu t & \alpha=1
\end{cases}
\]
\end{definition}
As the property 1.2.16 of \cite{samoradnitsky2017stable} shows
if $A \sim \text{St}(\alpha, \tau, \sigma,\mu)$ with $0<\alpha<2$ then $\E [|A|^r]< \infty$ for $0<r<\alpha$, and $\E [|A|^r]= \infty$ for $r \geq \alpha$.

\begin{definition} A $\R$-valued random variable $A$ is distributed as the symmetric $\alpha$-stable distribution with scale parameter $\sigma$, and we write $A \sim \text{St}(\alpha,\sigma)$, if it is stable with $\tau= \mu = 0$, i.e. the characteristic function is
\[
\varphi_A(t)= \E[e^{\ii tA}] = e^{-\sigma^{\alpha}|t|^{\alpha}}, \quad t \in \R
\]
\end{definition}
By property 1.2.3 of \cite{samoradnitsky2017stable}, $a\text{St}(\alpha, \sigma) = \sigma \text{St}(\alpha, |a|\sigma)$ for every $a\in \R$.
\begin{definition} A $\R^{D}$-valued random vector $A$ is distributed as the symmetric $D$-dimensional $\alpha$-stable distribution with scale (finite) spectral measure $\Gamma$ on $\mathbb{S}^{D-1}=\{z\in \R^{D}: \|z\|=1\}$, and we write $A \sim \text{St}_{D}(\alpha,\Gamma)$, if the characteristic function is
\[
\varphi_A(\textbf{t})= \E[e^{\ii \langle \textbf{t}, A \rangle}] = e^{-\int_{\mathbb{S}^{D-1}}|\langle \textbf{t},\textbf{s}\rangle|^{\alpha}\Gamma(\dd \textbf{s})}, \quad t \in \R^D
\]
\end{definition}

\begin{definition} A $\R^{\textbf{D}_S}$-valued random S-tensor $A$ is distributed as the symmetric $\textbf{D}_S$-dimensional $\alpha$-stable distribution with scale (finite) spectral measure $\Gamma$ on $\mathbb{S}^{|\textbf{D}_S|-1}=\{z\in \R^{|\textbf{D}_S|}: \|z\|=1\}$, and we write $A \sim \text{St}_{\textbf{D}_S}(\alpha,\Gamma)$, if the flattened tensor $A$ is distributed as $\text{St}_{|\textbf{D}_S|}(\alpha,\Gamma)$, i.e. if the characteristic function of $A$ is
\[
\varphi_A(t)= \E[e^{\ii \textbf{t} \otimes A }]=e^{-\int_{\mathbb{S}^{|\textbf{D}_S|-1}}| \textbf{t} \otimes \textbf{s} |^{\alpha}\Gamma(\dd \textbf{s})}, \quad t \in \R^{\textbf{D}_S}
\]
where $\dd \textbf{s}$ is considered flattened.
\end{definition}


\section*{SM: some useful inequalities}
During the proofs we will use without any mention the following inequality:
\begin{lemma}\label{base_inequality}
For any real values $\alpha, z_1, \dots z_n \geq 0$ there exists a constant $C=C(\alpha,n)$ such that 
\[
(z_1 + \dots +z_n)^{\alpha} \leq C (z_1^{\alpha} + \dots +z_n^{\alpha})
\]
\end{lemma}
\begin{proof}
Let $Z= \max \{z_1, \dots ,z_n \}$. Thus we get
\[
(z_1 + \dots +z_n)^{\alpha} \leq (nZ)^{\alpha}=n^{\alpha}Z^{\alpha}\leq n^{\alpha}(z_1^{\alpha} + \dots +z_n^{\alpha})
\]
In particular $C=C(\alpha,n)=n^{\alpha}$.
\end{proof}
We give an important intuition of the reason why the proofs that will follow work. Intuitively we will provide an alternative proof of the strong law of large numbers for Stable random variables using the de Finetti theorem regarding the exchangeability of sequences of random variables. To this end we will have to require that the expected value of the stochastic process is finite. Using the above Lemma \ref{base_inequality}, from assumption (\ref{envelope}) we get 
\[|\phi(s)|^{\alpha} \leq (a+b|s|^{\beta})^{\alpha}\leq 2^{\alpha}(a^{\alpha}+b^{\alpha}|s|^{\beta \alpha})
\]
When $s$ is $\alpha$-stable distributed with any skewness, scale and shift parameters, then
\[
\E[|\phi(s)|^{\alpha}] \leq 2^{\alpha}a^{\alpha}+2^{\alpha}b^{\alpha}\E[|s|^{\beta \alpha}]< \infty
\]
which is finite since $\beta<1$ thus $\beta \alpha < \alpha$. Then, the assumption $\beta <1$ is essential to guarantee the existence of the expected value of the stochastic process. However, this assumption also allows us to apply Jensen's inequality in the following sense: for any positive random variable $s$,
\[
\E[s^{\beta}] \leq (\E[s])^{\beta}
\]
We will use these inequalities repeatedly during the proofs.

\section*{SM: the patch operator}\zlabel{sec:app:patch}

Referring to the definition (\ref{eq:conv_expl_joint}) of DCNN we said that for each $l \in [L]$, $f^{(l-1)}_{\patch p^{(l)}} \in \R^{C^{(l-1)} \times \textbf{G}^{(l)}}$ is a patch extracted from $f^{(l-1)}$ at output position $p^{(l)} \in [\textbf{P}^{(l)}]$ and is defined by $f^{(l-1)}_{\patch p^{(l)}} = f^{(l-1)}_{(1:C^{(l-1)},\patch p^{(l)})}$. The patch operator is a map $\patch : [\textbf{P}^{(l)}] \to (\Z_1 \times \dots \times \Z_{S^{(l-1)}})^{\textbf{G}^{(l)}}$ that for each $p^{(l)}$, returns $\patch p^{(l)}$ that is a $S^{(l-1)}$ tensor of dimension $\textbf{G}^{(l)}$ containing some spatial indexes of $f^{(l-1)}$ depending on the moving window chosen in the structure (padding, stride etc). Refer to ``\ztitleref{sec:app:patch}" for a more detailed explanation.

More precisely, for each $g^{(l)} \in \textbf{G}^{(l)}$, $\patch p^{(l)}_{g^{(l)}} =(i_1, \dots , i_{S^{(l-1)}}) \in \Z_1 \times \dots \times \Z_{S^{(l-1)}}$ where we could have $i_{s^{(l-1)}} < 1$ or $i_{s^{(l-1)}} > P_{s^{(l-1)}}$ for some $s^{(l-1)} \in [S^{(l-1)}]$ because extraction could happens outside of $f^{(l-1)}$. Thus, two possibilities are allowed: 1) there exists $s^{(l-1)} \in [S^{(l-1)}]$ such that $i_{s^{(l-1)}} < 1$ or $i_{s^{(l-1)}} > P_{s^{(l-1)}}$. In that case $(f^{(l-1)}_{\patch p^{(l)}})_{(c^{(l-1)},g^{(l)})} = f^{(l-1)}_{(c^{(l-1)},\patch p^{(l)}_{g^{(l)}})}=0$ for each $c^{(l-1)} \in [C^{(l-1)}]$; 2) $i_{s^{(l-1)}} \in [P_{s^{(l-1)}}]$ for all $s^{(l-1)} \in [S^{(l-1)}]$. In that case there exist $p^{(l-1)} \in [\textbf{P}^{(l-1)}]$ such that $\patch p^{(l)}_{g^{(l)}} = p^{(l-1)}$, then $(f^{(l-1)}_{\patch p^{(l)}})_{(c^{(l-1)},g^{(l)})} = f^{(l-1)}_{(c^{(l-1)},p^{(l-1)})}=0$ for each $c^{(l-1)} \in [C^{(l-1)}]$.

\section*{SM A}\zlabel{sec:appA}
Here we compute the limit distribution of $f^{(l)(1:K)}_{(c^{(l)},:)}$ as $C \rightarrow \infty$. First we prove the following two theorems (see the corresponding appendix for the proof):

\begin{theorem}\label{theorem1} [\ztitleref{sec:appA.1}]
$f^{(1)(1:K)}_{(c^{(1)},:)} \sim \text{St}_{\textbf{P}^{(1)} \times K}(\alpha,\Gamma^{(1)})$, where 

\begin{align*}
    \Gamma^{(1)}=& \| \sigma_b \1_{(\textbf{P}^{(1)} \times K)} \|^{\alpha} \Psi^{(1)} ( \1_{(\textbf{P}^{(1)} \times K)} )\\
    &\quad + \sum_{(c^{(0)},g^{(1)}) \in [C^{(0)} \times \textbf{G}^{(1)}]} \|\sigma_{\omega} (x^{(1:K)}_{\patch})_{(c^{(0)},g^{(1)})}\|^{\alpha} \Psi^{(1)} \Big(  (x^{(1:K)}_{\patch})_{(c^{(0)},g^{(1)})}\Big)
\end{align*}

where $(x^{(1:K)}_{\patch})_{(c^{(0)},g^{(1)})}=\big[ (x^{(k)}_{\patch p^{(1)}})_{(c^{(0)},g^{(1)})} \big]_{\{(p^{(1)},k)\in [\textbf{P}^{(1)} \times K]\}}$.
\end{theorem}
\begin{theorem}\label{theorem2}[\ztitleref{sec:appA.2}]
For $l=2, \dots ,L$,  $f^{(l)(1:K)}_{(c^{(l)},:)}|f^{(l-1)(1:K)}_{(1:C,:)} \sim \text{St}_{\textbf{P}^{(l)} \times K}(\alpha,\Gamma_C^{(l)})$, where 

\begin{align*}
    \Gamma^{(l)}_C &= \| \sigma_b \1_{(\textbf{P}^{(l)} \times K)} \|^{\alpha} \Psi^{(l)} \Big( \1_{(\textbf{P}^{(l)} \times K)} \Big)\\
    & + \tfrac{1}{C}\sum_{(c^{(l-1)},g^{(l)}) \in [C \times \textbf{G}^{(l)}]} \|\sigma_{\omega} \phi (f^{(l-1)(1:K)}_{\patch})_{(c^{(l-1)},g^{(l)})}\|^{\alpha} \Psi^{(l)} \Big( \phi (f^{(l-1)(1:K)}_{\patch})_{(c^{(l-1)},g^{(l)})} \Big)
\end{align*}

with $ (f^{(l-1)(1:K)}_{\patch})_{(c^{(l-1)},g^{(l)})}=\Big[ (f^{(l-1)(k)}_{\patch p^{(l)}})_{(c^{(l-1)},g^{(l)})} \Big]_{\{(p^{(l)},k)\in [\textbf{P}^{(l)} \times K]\}}$.
\end{theorem}

Now, for $l \in [L]$ we compute the limit distribution of $f^{(l)(1:K)}_{(c^{(l)},:)}$ as $C \rightarrow \infty$. First recall that for $l=1$, being $f^{(1)(1:K)}_{(c^{(1)},:)}$ referred only to the $C^{(0)}$ channels of the input layer, then $f^{(1)(1:K)}_{(c^{(1)},:)} \overset{d}{\rightarrow} \text{St}_{\textbf{P}^{(1)} \times K}(\alpha,\Gamma^{(1)}_{\infty})$ constantly as $C \rightarrow \infty$, where we have defined $\Gamma^{(1)}_{\infty}=\Gamma^{(1)}$. Thus we compute the limit for all the others layers. We prove that for each $l = 2, \dots, L$,  $f^{(l)(1:K)}_{(c^{(l)},:)} \overset{d}{\rightarrow} \text{St}_{\textbf{P}^{(l)} \times K}(\alpha,\Gamma^{(l)}_{\infty})$ as $C \rightarrow \infty$, where
\[
\begin{split}
    \Gamma^{(l)}_{\infty} = & \| \sigma_b \1_{(\textbf{P}^{(l)} \times K)} \|^{\alpha} \Psi^{(l)} \Big( \1_{(\textbf{P}^{(l)} \times K)} \Big)+ \\
    &+ \int \sum_{g^{(l)} \in [\textbf{G}^{(l)}]} \|\sigma_{\omega} \phi (f_{g^{(l)}})\|^{\alpha} \Psi^{(l)} \Big( \phi (f_{g^{(l)}}) \Big) q^{(l-1)}(\dd f_{\{g^{(l)}\in [\textbf{G}^{(l)}]\}})
\end{split}
\]
with $f_{g^{(l)}} \in \R^{\textbf{P}^{(l)} \times K}$ for each $g^{(l)} \in \textbf{G}^{(l)}$ and $q^{(l-1)} = \text{St}_{\textbf{P}^{(l-1)} \times K}(\alpha,\Gamma^{(l-1)}_{\infty})$.

For the proof we will need the following proposition which is a direct consequence of Exercise 2.3.4 of \cite{samoradnitsky2017stable}.
\begin{prop}\label{prop1} If $A \sim \text{St}_{D}(\alpha,\Gamma)$ then for each $\textbf{u} \in \R^D$ the $1$-dimensional r.v. $\langle \textbf{u},A \rangle \sim \text{St}(\alpha, \tau(\textbf{u}), \sigma(\textbf{u}), \mu(\textbf{u}))$ where
\[
\begin{split}
&\sigma(\textbf{u}) = \Big( \int_{\mathbb{S}^{D-1}}|\langle \textbf{u},\textbf{s}\rangle|^{\alpha}\Gamma(\dd \textbf{s}) \Big)^{1/\alpha}\\
&\tau(\textbf{u}) = \sigma(\textbf{u})^{-1} \int_{\mathbb{S}^{D-1}}|\langle \textbf{u},\textbf{s}\rangle|^{\alpha} \text{sign}(\langle \textbf{u},\textbf{s}\rangle)\Gamma(\dd \textbf{s})\\
& \mu(\textbf{u}) = \begin{cases}
 0 & \alpha \neq 1\\
 -\frac{2}{\pi}\int_{\mathbb{S}^{D-1}}\langle \textbf{u},\textbf{s}\rangle \log( |\langle \textbf{u},\textbf{s}\rangle|)\Gamma(\dd \textbf{s}) &  \alpha =1
\end{cases}
\end{split}
\]
\end{prop}

\begin{proof}
Fix $l=2, \dots, L$ and, for each $C$, let $h^{(l)}_C$ denote the de Finetti random probability measure of the exchangeable sequence $((f^{(l)(1:K)}_{\patch})_{(c^{(l)},:)})_{c^{(l)} \geq 1}$, i.e. $(f^{(l)(1:K)}_{\patch})_{(c^{(l)},:)}|h^{(l)}_C \overset{iid}{\sim} h^{(l)}_C$.  Consider the induction hypothesis that as $C \rightarrow \infty$ 

\[
h^{(l-1)}_C \overset{w}{\rightarrow} q^{(l-1)}
\]

where $q^{(l-1)}=\text{St}_{\textbf{P}^{(l-1)} \times K}(\alpha,\Gamma^{(l-1)}_{\infty})$ and the finite measure $\Gamma^{(l-1)}_{\infty}$ will be specified.
For $l>1$ and any $\textbf{t}^{(l)}:=[t^{(l)(k)}_{p^{(l)}}]_{\{(p^{(l)},k) \in [\textbf{P}^{(l)} \times K] \}} \in \R^{\textbf{P}^{(l)} \times K}$,

\begin{equation}\label{eq:conditioned}
\begin{split}
&\varphi_{\big( f^{(l)(1:K)}_{(c^{(l)},:)}\big)} (\textbf{t}^{(l)})\\
=& \E \Big[ \exp \Big\{ \ii \textbf{t}^{(l)}\otimes f^{(l)(1:K)}_{(c^{(l)},:)} \Big\} \Big]\\
=& \E \Big[ \E \Big[ \exp \Big\{ \ii \textbf{t}^{(l)}\otimes f^{(l)(1:K)}_{(c^{(l)},:)} \Big\} \Big|f^{(l-1)(1:K)}_{(1:C,:)} \Big] \Big]\\
=& \E \Big[ \exp \Big\{ - \int_{\mathbb{S}^{|\textbf{P}^{(l)}\times K|-1}} | \textbf{t}^{(l)} \otimes s^{(l)}|^{\alpha} \Gamma^{(l)}_C (\dd s^{(l)}) \Big\} \Big]\\
=& \exp \Big\{ -\sigma^{\alpha}_b|\textbf{t}^{(l)} \otimes  \1_{(\textbf{P}^{(l)} \times K)}|^{\alpha} \Big\} \times \\
\quad & \times \E \Big[ \exp \Big\{ -\frac{\sigma^{\alpha}_{\omega}}{C}\sum_{(c^{(l-1)},g^{(l)}) \in [C \times \textbf{G}^{(l)}]} | \textbf{t}^{(l)} \otimes \phi (f^{(l-1)(1:K)}_{\patch})_{(c^{(l-1)},g^{(l)})} |^{\alpha} \Big\} \Big]\\
=& \exp \Big\{ -\sigma^{\alpha}_b|\textbf{t}^{(l)} \otimes  \1_{(\textbf{P}^{(l)} \times K)}|^{\alpha} \Big\} \times \\
\quad & \times \E \Big[ \E \Big[ \exp \Big\{ -\frac{\sigma^{\alpha}_{\omega}}{C}\sum_{(c^{(l-1)},g^{(l)}) \in [C \times \textbf{G}^{(l)}]} | \textbf{t}^{(l)} \otimes \phi (f^{(l-1)(1:K)}_{\patch})_{(c^{(l-1)},g^{(l)})} |^{\alpha} \Big\} \Big] \Big| h^{(l-1)}_{C} \Big]\\
=& \exp \Big\{ -\sigma^{\alpha}_b|\textbf{t}^{(l)} \otimes  \1_{(\textbf{P}^{(l)} \times K)}|^{\alpha} \Big\} \times \\
\quad & \times \E \Big[ \Big( \int \exp \Big\{ -\frac{\sigma^{\alpha}_{\omega}}{C} \sum_{g^{(l)} \in [\textbf{G}^{(l)}]}| \textbf{t}^{(l)} \otimes \phi (f_{g^{(l)}}) |^{\alpha} \Big\} h^{(l-1)}_C(\dd f_{\{g^{(l)}\in [\textbf{G}^{(l)}]\}}) \Big)^{C} \Big]
\end{split}
\end{equation}
where $f_{g^{(l)}}\in \R^{\textbf{P}^{(l)} \times K}$ for each $g^{(l)} \in \textbf{G}^{(l)}$. Hereafter we show the limiting behaviour. In order to do this we need the following lemmas:

\begin{itemize}
    \item[L1)] For each $l=2,\dots, L$,
    $\sup_C \int \sum_{g^{(l)} \in [\textbf{G}^{(l)}]} \|\phi(f_{g^{(l)}})\|^{\alpha} h_C^{(l-1)}(\dd f_{g^{(l)} \in [\textbf{G}^{(l)}]}) < \infty $
    
    \item[L1.1)] There exists $\epsilon >0$ such that, $\sup_C \E[\sum_{g^{(l)} \in [\textbf{G}^{(l)}]} \|\phi(f^{(l-1)(1:K)}_{\patch})_{(c^{(l-1)},g^{(l)})}\|^{\alpha+\epsilon} | h^{(l-2)}_C ] < \infty$ for each $l=2,\dots, L$
    
    \item[L2)] $\int \sum_{g^{(l)} \in [\textbf{G}^{(l)}]} |\textbf{t}^{(l)} \otimes \phi(f_{g^{(l)}})|^{\alpha} h^{(l-1)}_C(\dd f_{g^{(l)} \in [\textbf{G}^{(l)}]}) \overset{p}{\rightarrow} \int \sum_{g^{(l)} \in [\textbf{G}^{(l)}]} |\textbf{t}^{(l)} \otimes \phi(f_{g^{(l)}})|^{\alpha} q^{(l-1)}(\dd f_{g^{(l)} \in [\textbf{G}^{(l)}]})$ as $C \rightarrow \infty$
    
    \item [L3)] $\int \sum_{g^{(l)} \in [\textbf{G}^{(l)}]} \|\phi(f_{g^{(l)}})\|^{\alpha} [1-\exp\{-\frac{\sigma^{\alpha}_{\omega}}{C} \sum_{g^{(l)} \in [\textbf{G}^{(l)}]}| \textbf{t}^{(l)} \otimes \phi (f_{g^{(l)}}) |^{\alpha}\}] h^{(l-1)}_C(\dd f_{g^{(l)} \in [\textbf{G}^{(l)}]}) \overset{p}{\rightarrow} 0$ as $C \rightarrow \infty$
\end{itemize}

\subsection*{Proof of L1}
For $l=2$, for each $c^{(1)} \geq 1$, from assumptions (\ref{eq:param_dist2}) and (\ref{envelope}) and from Lemma \ref{base_inequality} we get
\[
\begin{split}
 &\E \Big[ \sum_{g^{(2)} \in [\textbf{G}^{(2)}]} \| \phi(f^{(1)(1:K)}_{\patch})_{(c^{(1)},g^{(2)})}\|^{\alpha} \Big] \\
 & = \E \Big[ \sum_{g^{(2)} \in [\textbf{G}^{(2)}]} \sum_{p^{(2)} \in [\textbf{P}^{(2)}]} \sum_{k \in [K]} | \phi(f^{(1)(k)}_{\patch p^{(2)}})_{(c^{(1)},g^{(2)})}|^{\alpha} \Big] \\
 & = \sum_{g^{(2)} \in [\textbf{G}^{(2)}]} \sum_{p^{(2)} \in [\textbf{P}^{(2)}]} \sum_{k \in [K]} \E \Big[ | \phi(f^{(1)(k)}_{\patch p^{(2)}})_{(c^{(1)},g^{(2)})}|^{\alpha} \Big] \\
 & \leq \sum_{g^{(2)} \in [\textbf{G}^{(2)}]} \sum_{p^{(2)} \in [\textbf{P}^{(2)}]} \sum_{k \in [K]} \E \Big[ ( a + b|(f^{(1)(k)}_{\patch p^{(2)}})_{(c^{(1)},g^{(2)})}|^{\beta})^{\alpha} \Big] \\
 & \leq 2^{\alpha} \sum_{g^{(2)} \in [\textbf{G}^{(2)}]} \sum_{p^{(2)} \in [\textbf{P}^{(2)}]} \sum_{k \in [K]} \E \Big[  a^{ \alpha} + b^{\alpha}|(f^{(1)(k)}_{\patch p^{(2)}})_{(c^{(1)},g^{(2)})}|^{ \alpha \beta} \Big] \\
 & = 2^{\alpha}|\textbf{G}^{(2)}| |\textbf{P}^{(2)}| K a^{\alpha} + 2^{\alpha}b^{\alpha} \sum_{g^{(2)} \in [\textbf{G}^{(2)}]} \sum_{p^{(2)} \in [\textbf{P}^{(2)}]} \sum_{k \in [K]}\E \Big[ |(f^{(1)(k)}_{\patch p^{(2)}})_{(c^{(1)},g^{(2)})}|^{\alpha \beta} \Big] \\
 & < \infty
\end{split}
\]

where we used that $(f^{(1)(k)}_{\patch p^{(2)}})_{(c^{(1)},g^{(2)})}$, by Proposition \ref{prop1} is distributed according to a SD with index $\alpha$ (and some skewness, scale and shift parameters) and then, being $\alpha \beta < \alpha$, $\E \Big[ |(f^{(1)(k)}_{\patch p^{(2)}})_{(c^{(1)},g^{(2)})}|^{ \alpha \beta} \Big]< +\infty$ . Now assuming that L1) is true for $l-2$ we prove that it is true for $l-1$. First, from assumptions (\ref{eq:param_dist2}) and (\ref{envelope}) and from Lemma \ref{base_inequality}, we compute the following

\[
\begin{split}
& \E \Big[ \sum_{g^{(l)} \in [\textbf{G}^{(l)}]} \| \phi(f^{(l-1)(1:K)}_{\patch})_{(c^{(l-1)},g^{(l)})}\|^{\alpha} \Big| f^{(l-2)(1:K)}_{(1:C,:)}\Big] \\
 &= \E \Big[ \sum_{g^{(l)} \in [\textbf{G}^{(l)}]} \sum_{p^{(l)} \in [\textbf{P}^{(l)}]} \sum_{k \in [K]} | \phi(f^{(l-1)(k)}_{\patch p^{(l)}})_{(c^{(l-1)},g^{(l)})} |^{\alpha} \Big| f^{(l-2)(1:K)}_{(1:C,:)} \Big] \\
 &= \sum_{g^{(l)} \in [\textbf{G}^{(l)}]} \sum_{p^{(l)} \in [\textbf{P}^{(l)}]} \sum_{k \in [K]} \E \Big[ | \phi(f^{(l-1)(k)}_{\patch p^{(l)}})_{(c^{(l-1)},g^{(l)})} |^{\alpha} \Big| f^{(l-2)(1:K)}_{(1:C,:)} \Big] \\
 & \leq \sum_{g^{(l)} \in [\textbf{G}^{(l)}]} \sum_{p^{(l)} \in [\textbf{P}^{(l)}]} \sum_{k \in [K]} \E \Big[ ( a + b | (f^{(l-1)(k)}_{\patch p^{(l)}})_{(c^{(l-1)},g^{(l)})} |^{\beta} )^{ \alpha}\Big| f^{(l-2)(1:K)}_{(1:C,:)} \Big] \\
 & \leq 2^{\alpha}|\textbf{G}^{(l)}| |\textbf{P}^{(l)}| K a^{ \alpha} + 2^{ \alpha} b^{ \alpha} \sum_{g^{(l)} \in [\textbf{G}^{(l)}]} \sum_{p^{(l)} \in [\textbf{P}^{(l)}]} \sum_{k \in [K]} \E \Big[ |(f^{(l-1)(k)}_{\patch p^{(l)}})_{(c^{(l-1)},g^{(l)})} |^{\beta \alpha}\Big| f^{(l-2)(1:K)}_{(1:C,:)} \Big]
 \end{split} 
\]

Recall that for each $p^{(l)} \in [\textbf{P}^{(l)}]$, $(f^{(l-1)(k)}_{\patch p^{(l)}})_{(c^{(l-1)},g^{(l)})}$ could be equal to $0$ or there exists an unique position $p^{(l-1)} \in [\textbf{P}^{(l-1)}]$ such that  $(f^{(l-1)(k)}_{\patch p^{(l)}})_{(c^{(l-1)},g^{(l)})} = f^{(l-1)(k)}_{(c^{(l-1)},p^{(l-1)})}$, thus we get

\[
\begin{split}
& \E \Big[ \sum_{g^{(l)} \in [\textbf{G}^{(l)}]} \| \phi(f^{(l-1)(1:K)}_{\patch})_{(c^{(l-1)},g^{(l)})}\|^{\alpha} \Big| f^{(l-2)(1:K)}_{(1:C,:)}\Big] \\
 & \leq 2^{ \alpha} |\textbf{G}^{(l)}| |\textbf{P}^{(l)}| K a^{ \alpha} + 2^{ \alpha}b^{ \alpha} \sum_{p^{(l-1)} \in [\textbf{P}^{(l-1)}]} \sum_{k \in [K]} \E \Big[ |f^{(l-1)(k)}_{(c^{(l-1)},p^{(l-1)})} |^{\beta \alpha}\Big| f^{(l-2)(1:K)}_{(1:C,:)} \Big]
 \end{split} 
\]

Moreover, from theorem \ref{theorem2} we know that  $f^{(l-1)(1:K)}_{(c^{(l-1)},:)} | f^{(l-2)(1:K)}_{(1:C,:)} \sim \text{St}_{\textbf{P}^{(l-1)} \times K}(\alpha,\Gamma^{(l-1)})$ and from proposition \ref{prop1}, denoted $U(p^{(l-1)},k) = \1_{(\textbf{P}^{(l-1)} \times K)}[(p^{(l-1)},k)]$, for each $(p^{(l-1)},k) \in [\textbf{P}^{(l-1)} \times K]$ we have $f^{(l-1)(k)}_{(c^{(l-1)},p^{(l-1)})} | f^{(l-2)(1:K)}_{(1:C,:)} = U(p^{(l-1)},k) \otimes f^{(l-1)(1:K)}_{(c^{(l-1)},:)} | f^{(l-2)(1:K)}_{(1:C,:)} \sim \text{St}\Big(\alpha,\tau(U(p^{(l-1)},k)), \sigma(U(p^{(l-1)},k)), \mu (U(p^{(l-1)},k))\Big)\sim \sigma(U(p^{(l-1)},k)) \text{St}\Big(\alpha,\tau(U(p^{(l-1)},k)), 1 , \mu (U(p^{(l-1)},k))\Big)$. Since $\beta \alpha < \alpha$ we get
\[
\begin{split}
& \E \Big[ \sum_{g^{(l)} \in [\textbf{G}^{(l)}]} \| \phi(f^{(l-1)(1:K)}_{\patch})_{(c^{(l-1)},g^{(l)})}\|^{\alpha} \Big| f^{(l-2)(1:K)}_{(1:C,:)}\Big] \\
& \leq 2^{\alpha}|\textbf{G}^{(l)}| |\textbf{P}^{(l)}| K a^{ \alpha} +\\
& \quad + 2^{ \alpha}b^{ \alpha} \sum_{p^{(l-1)} \in [\textbf{P}^{(l-1)}]} \sum_{k \in [K]} \sigma\Big( U(p^{(l-1)},k) \Big)^{\beta \alpha} \underset{<\infty}{\underbrace{\E \Big[ |\text{St}(\alpha,\tau (U(p^{(l-1)},k)),1,\mu (U(p^{(l-1)},k)))|^{\beta \alpha}\Big]}}\\
& \leq 2^{ \alpha}|\textbf{G}^{(l)}| |\textbf{P}^{(l)}| K a^{ \alpha} +\\
&\quad + 2^{ \alpha}b^{\alpha} \mathcal{M} \sum_{p^{(l-1)} \in [\textbf{P}^{(l-1)}]} \sum_{k \in [K]} \Bigg( \int_{\mathbb{S}^{|\textbf{P}^{(l-1)}\times K|-1}} \Big|  U(p^{(l-1)},k) \otimes s \Big|^{\alpha} \Gamma^{(l-1)}(\dd s) \Bigg)^{\beta }
\end{split} 
\]

where $\mathcal{M}= \max _{(p^{(l-1)},k) \in [\textbf{P}^{(l-1)} \times k]} \E \Big[ |\text{St}(\alpha,\tau (U(p^{(l-1)},k)),1,\mu (U(p^{(l-1)},k)))|^{\beta \alpha}\Big] < + \infty$. Then,

\begin{equation}\label{equationL1A}
\begin{split}
&\E \Big[ \sum_{g^{(l)} \in [\textbf{G}^{(l)}]} \| \phi(f^{(l-1)(1:K)}_{\patch})_{(c^{(l-1)},g^{(l)})}\|^{\alpha} \Big| h_C^{(l-2)}\Big] \\
&= \E \Big[\E \Big[ \sum_{g^{(l)} \in [\textbf{G}^{(l)}]} \| \phi(f^{(l-1)(1:K)}_{\patch})_{(c^{(l-1)},g^{(l)})}\|^{\alpha} \Big| f^{(l-2)(1:K)}_{(1:C,:)}\Big]\Big| h_C^{(l-2)}\Big] \\
& \leq 2^{\alpha}|\textbf{G}^{(l)}| |\textbf{P}^{(l)}| K a^{ \alpha} +\\
& \quad + 2^{\alpha}b^{\alpha} \mathcal{M} \sum_{p^{(l-1)} \in [\textbf{P}^{(l-1)}]} \sum_{k \in [K]} \E\Bigg[ \Bigg( \int_{\mathbb{S}^{|\textbf{P}^{(l-1)}\times K|-1}} \Big|  U(p^{(l-1)},k) \otimes s \Big|^{\alpha} \Gamma^{(l-1)}(\dd s) \Bigg)^{\beta}\Big| h_C^{(l-2)} \Bigg]\\
& \leq 2^{ \alpha}|\textbf{G}^{(l)}| |\textbf{P}^{(l)}| K a^{\alpha} +\\
& \quad + 2^{ \alpha}b^{ \alpha} \mathcal{M} \sum_{p^{(l-1)} \in [\textbf{P}^{(l-1)}]} \sum_{k \in [K]} \Bigg(\E\Bigg[ \int_{\mathbb{S}^{|\textbf{P}^{(l-1)}\times K|-1}} \Big|  U(p^{(l-1)},k) \otimes s \Big|^{\alpha} \Gamma^{(l-1)}(\dd s)\Big| h_C^{(l-2)} \Bigg] \Bigg)^{\beta }
\end{split} 
\end{equation}
where we used the Jensen's inequality. Moreover,

\begin{equation}\label{equazioneL1B}
\begin{split}
&\E\Bigg[ \int_{\mathbb{S}^{|\textbf{P}^{(l-1)}\times K|-1}} \Big|  U(p^{(l-1)},k) \otimes s \Big|^{\alpha} \Gamma^{(l-1)}(\dd s)\Bigg| h_C^{(l-2)} \Bigg]\\
&=\E\Bigg[ \sigma_b^{\alpha}|U(p^{(l-1)},k) \otimes \1_{(\textbf{P}^{(l-1)} \times K)}|^{\alpha} + \\
&\qquad +\frac{\sigma_{\omega}^{\alpha}}{C}\sum_{(c^{(l-2)},g^{(l-1)}) \in [C \times \textbf{G}^{(l-1)}]} \Big|U(p^{(l-1)},k) \otimes \phi(f^{(l-2)(1:K)}_{\patch})_{(c^{(l-2)},g^{(l-1)})} \Big|^{\alpha}\Bigg| h_C^{(l-2)}  \Bigg]\\
&=\E\Bigg[ \sigma_b^{\alpha} + \frac{\sigma_{\omega}^{\alpha}}{C}\sum_{(c^{(l-2)},g^{(l-1)}) \in [C \times \textbf{G}^{(l-1)}]} | \phi(f^{(l-2)(k)}_{\patch p^{(l-1)}})_{(c^{(l-2)},g^{(l-1)})} |^{\alpha}\Bigg| h_C^{(l-2)}  \Bigg]\\
&=\sigma_b^{\alpha} + \frac{\sigma_{\omega}^{\alpha}}{C}\sum_{(c^{(l-2)},g^{(l-1)}) \in [ C \times \textbf{G}^{(l-1)}]} \E\Big[| \phi(f^{(l-2)(k)}_{\patch p^{(l-1)}})_{(c^{(l-2)},g^{(l-1)})} |^{\alpha} \Big| h_C^{(l-2)}  \Big]\\
&=\sigma_b^{\alpha} + \frac{\sigma_{\omega}^{\alpha}}{C}\sum_{(c^{(l-2)},g^{(l-1)}) \in [C \times \textbf{G}^{(l-1)}]} \E\Big[\| \phi(f^{(l-2)(1:K)}_{\patch })_{(c^{(l-2)},g^{(l-1)})} \|^{\alpha} \Big| h_C^{(l-2)}  \Big]\\
&=\sigma_b^{\alpha} + \frac{\sigma_{\omega}^{\alpha}}{C}\sum_{(c^{(l-2)},g^{(l-1)}) \in [C \times \textbf{G}^{(l-1)}]} \int \| \phi(f_{g^{(l-1)}})\|^{\alpha} h_C^{(l-2)}(\dd f_{g^{(l-1)} \in [\textbf{G}^{(l-1)}]} ) \\
&=\sigma_b^{\alpha} + \sigma_{\omega}^{\alpha} \int \sum_{g^{(l-1)} \in [\textbf{G}^{(l-1)}]} \| \phi(f_{g^{(l-1)}})\|^{\alpha} h_C^{(l-2)}(\dd f_{g^{(l-1)} \in [\textbf{G}^{(l-1)}]} )
\end{split}
\end{equation}
Note that we have used the inequality $|x_i| \leq \|x\|$ and that $(f^{(l-2)(1:K)}_{\patch})_{(c^{(l-2)},:)}|h^{(l-2)}_C \overset{iid}{\sim} h^{(l-2)}_C$ (with respect to $c^{(l-1)} \geq 1$). Putting together (\ref{equationL1A}) and (\ref{equazioneL1B}) we have shown that

\[
\begin{split}
&\sup_C \E \Big[ \sum_{g^{(l)} \in [\textbf{G}^{(l)}]} \| \phi(f^{(l-1)(1:K)}_{\patch})_{(c^{(l-1)},g^{(l)})}\|^{\alpha} \Big| h_C^{(l-2)}\Big]\\
&\leq 2^{\alpha}|\textbf{G}^{(l)}| |\textbf{P}^{(l)}| K a^{ \alpha} + 2^{ \alpha}b^{ \alpha} \mathcal{M} \sum_{p^{(l-1)} \in [\textbf{P}^{(l-1)}]} \sum_{k \in [K]} \Bigg( \sigma_b^{\alpha} \\
&\quad + \sigma_{\omega}^{\alpha} \sup_C \int \sum_{g^{(l-1)} \in [\textbf{G}^{(l-1)}]} \| \phi(f_{g^{(l-1)}})\|^{\alpha} h_C^{(l-2)}(\dd f_{g^{(l-1)} \in [\textbf{G}^{(l-1)}]} ) \Bigg)^{\beta }
\end{split}
\]
which is finite by induction hypothesis. Now we conclude:

\[
\begin{split}
&\sup_C \int \sum_{g^{(l)} \in [\textbf{G}^{(l)}]} \|\phi(f_{g^{(l)}})\|^{\alpha} h_C^{(l-1)}(\dd f_{g^{(l)} \in [\textbf{G}^{(l)}]}) \\
&= \sup_C\E \Big[ \sum_{g^{(l)} \in [\textbf{G}^{(l)}]} \|\phi(f^{(l-1)(1:K)}_{\patch})_{(c^{(l-1)},g^{(l)})}\|^{\alpha} \Big|h^{(l-1)}_C \Big]\\
&= \sup_C \E \Big[ \sup_C\E \Big[ \sum_{g^{(l)} \in [\textbf{G}^{(l)}]} \|\phi(f^{(l-1)(1:K)}_{\patch})_{(c^{(l-1)},g^{(l)})}\|^{\alpha} \Big|h^{(l-2)}_C \Big]\Big|h^{(l-1)}_C \Big]\\
&< +\infty
\end{split}
\]
that is finite by previous step.

\subsection*{Proof of L1.1}
The proof of L1.1) follows by induction, and along lines similar to the proof of L1). In particular, let $\epsilon>0$ be such that $\beta (\alpha + \epsilon)<\alpha$. It exists since $\beta <1$. For $l=2$, for each $c^{(1)} \geq 1$
\[
\begin{split}
& \E \Big[ \sum_{g^{(2)} \in [\textbf{G}^{(2)}]} \| \phi(f^{(1)(1:K)}_{\patch})_{(c^{(1)},g^{(2)})}\|^{\alpha+\epsilon} \Big] \\
& = \E \Big[ \sum_{g^{(2)} \in [\textbf{G}^{(2)}]} \sum_{p^{(2)} \in [\textbf{P}^{(2)}]} \sum_{k \in [K]} | \phi(f^{(1)(k)}_{\patch p^{(2)}})_{(c^{(1)},g^{(2)})}|^{\alpha+\epsilon} \Big] \\
 & = \sum_{g^{(2)} \in [\textbf{G}^{(2)}]} \sum_{p^{(2)} \in [\textbf{P}^{(2)}]} \sum_{k \in [K]} \E \Big[ | \phi(f^{(1)(k)}_{\patch p^{(2)}})_{(c^{(1)},g^{(2)})}|^{\alpha+\epsilon} \Big] \\
 & \leq \sum_{g^{(2)} \in [\textbf{G}^{(2)}]} \sum_{p^{(2)} \in [\textbf{P}^{(2)}]} \sum_{k \in [K]} \E \Big[ ( a + b|(f^{(1)(k)}_{\patch p^{(2)}})_{(c^{(1)},g^{(2)})}|^{\beta})^{ (\alpha+\epsilon)} \Big] \\
 & \leq \sum_{g^{(2)} \in [\textbf{G}^{(2)}]} \sum_{p^{(2)} \in [\textbf{P}^{(2)}]} \sum_{k \in [K]} \E \Big[  a^{ (\alpha+\epsilon)} + b^{ (\alpha+\epsilon)}|(f^{(1)(k)}_{\patch p^{(2)}})_{(c^{(1)},g^{(2)})}|^{ (\alpha+\epsilon) \beta} \Big] \\
 & \leq 2^{ (\alpha+\epsilon)}|\textbf{G}^{(2)}| |\textbf{P}^{(2)}| K a^{ (\alpha+\epsilon)} +(2b)^{ (\alpha+\epsilon)} \sum_{g^{(2)} \in [\textbf{G}^{(2)}]} \sum_{p^{(2)} \in [\textbf{P}^{(2)}]} \sum_{k \in [K]}\E \Big[ |(f^{(1)(k)}_{\patch p^{(2)}})_{(c^{(1)},g^{(2)})}|^{ (\alpha+\epsilon) \beta} \Big] \\
 & < \infty
\end{split}
\]

where we used that $(f^{(1)(k)}_{\patch p^{(2)}})_{(c^{(1)},g^{(2)})}$, by Proposition \ref{prop1} is distributed according to a SD with index $\alpha$ (and some skewness, scale and shift parameters) and then, being $ (\alpha+\epsilon) \beta < \alpha$, $\E \Big[ |(f^{(1)(k)}_{\patch p^{(2)}})_{(c^{(1)},g^{(2)})}|^{ (\alpha+\epsilon) \beta} \Big]< +\infty$. Moreover the bound is uniform with respect to $C$ since the law is invariant with respect to $C$. Now assuming that L1.1) is true for $l-2$ we prove that it is true for $l-1$. As in the previous lemma, we can write the following inequality
\begin{equation}\label{eq:unif_integ}
\begin{split}
&\E \Big[ \sum_{g^{(l)} \in [\textbf{G}^{(l)}]} \| \phi(f^{(l-1)(1:K)}_{\patch})_{(c^{(l-1)},g^{(l)})}\|^{\alpha+\epsilon} \Big| h_C^{(l-2)}\Big] \\
&= \E \Big[\E \Big[ \sum_{g^{(l)} \in [\textbf{G}^{(l)}]} \| \phi(f^{(l-1)(1:K)}_{\patch})_{(c^{(l-1)},g^{(l)})}\|^{\alpha+\epsilon} \Big| f^{(l-2)(1:K)}_{(1:C,:)}\Big]\Big| h_C^{(l-2)}\Big] \\
& \leq 2^{(\alpha+\epsilon)}|\textbf{G}^{(l)}| |\textbf{P}^{(l)}| K a^{ (\alpha+\epsilon)} +2^{(\alpha+\epsilon)}b^{(\alpha+\epsilon)} \mathcal{M} \sum_{p^{(l-1)} \in [\textbf{P}^{(l-1)}]} \\
& \quad \sum_{k \in [K]} \E\Bigg[ \Bigg( \int_{\mathbb{S}^{|\textbf{P}^{(l-1)}\times K|-1}} \Big|  U(p^{(l-1)},k) \otimes s \Big|^{\alpha+\epsilon} \Gamma^{(l-1)}(\dd s) \Bigg)^{\beta}\Big| h_C^{(l-2)} \Bigg]\\
& \leq 2^{(\alpha+\epsilon)}|\textbf{G}^{(l)}| |\textbf{P}^{(l)}| K a^{(\alpha+\epsilon)} + 2^{(\alpha+\epsilon)} b^{ (\alpha+\epsilon)} \mathcal{M} \sum_{p^{(l-1)} \in [\textbf{P}^{(l-1)}]} \\
& \quad \sum_{k \in [K]} \Bigg(\E\Bigg[ \int_{\mathbb{S}^{|\textbf{P}^{(l-1)}\times K|-1}} \Big|  U(p^{(l-1)},k) \otimes s \Big|^{\alpha+\epsilon} \Gamma^{(l-1)}(\dd s)\Big| h_C^{(l-2)} \Bigg] \Bigg)^{\beta }
\end{split} 
\end{equation}
Moreover, following the same steps as in the previous lemma (just replacing $\alpha+\epsilon$ instead of $\alpha$), we get
\[
\begin{split}
& \E\Bigg[ \int_{\mathbb{S}^{|\textbf{P}^{(l-1)}\times K|-1}} \Big|  U(p^{(l-1)},k) \otimes s \Big|^{\alpha+\epsilon} \Gamma^{(l-1)}(\dd s)\Bigg| h_C^{(l-2)} \Bigg]\\
&=\sigma_b^{\alpha+\epsilon} + \sigma_{\omega}^{\alpha+\epsilon} \int \sum_{g^{(l-1)} \in [\textbf{G}^{(l-1)}]} \| \phi(f_{g^{(l-1)}})\|^{\alpha+\epsilon} h_C^{(l-2)}(\dd f_{g^{(l-1)} \in [\textbf{G}^{(l-1)}]} )
\end{split}
\]
and
\[
\begin{split}
&\int \sum_{g^{(l-1)} \in [\textbf{G}^{(l-1)}]} \| \phi(f_{g^{(l-1)}})\|^{\alpha+\epsilon} h_C^{(l-2)}(\dd f_{g^{(l-1)} \in [\textbf{G}^{(l-1)}]} )\\
&= \E\Big[ \sum_{g^{(l-1)} \in [\textbf{G}^{(l-1)}]} \| \phi(f^{(l-2)(1:K)}_{\patch})_{(c^{(l-2)},g^{(l-1)})}\|^{\alpha+\epsilon} \Big| h_C^{(l-2)} \Big]\\
&=\E\Big[ \E\Big[ \sum_{g^{(l-1)} \in [\textbf{G}^{(l-1)}]} \| \phi(f^{(l-2)(1:K)}_{\patch})_{(c^{(l-2)},g^{(l-1)})}\|^{\alpha+\epsilon} \Big| h_C^{(l-3)} \Big] \Big| h_C^{(l-2)} \Big]\\
&\leq\E\Big[ \E\Big[ \sum_{g^{(l-1)} \in [\textbf{G}^{(l-1)}]} \| \phi(f^{(l-2)(1:K)}_{\patch})_{(c^{(l-2)},g^{(l-1)})}\|^{\alpha+\epsilon} \Big| h_C^{(l-3)} \Big] \Big| h_C^{(l-2)} \Big]
\end{split}
\]

Thus taking the $\sup_C$ in (\ref{eq:unif_integ}), by previous inequalities, it is less or equal than

\[
\begin{split}
&2^{(\alpha+\epsilon)}|\textbf{G}^{(l)}| |\textbf{P}^{(l)}| K a^{ (\alpha+\epsilon)} + 2^{(\alpha+\epsilon)}b^{ (\alpha+\epsilon)} \mathcal{M} \sum_{p^{(l-1)} \in [\textbf{P}^{(l-1)}]} \sum_{k \in [K]} \Bigg( \sigma_b^{\alpha+\epsilon} +\\
&+ \sigma_{\omega}^{\alpha+\epsilon} \E\Big[ \sup_C \E\Big[ \sum_{g^{(l-1)} \in [\textbf{G}^{(l-1)}]} \| \phi(f^{(l-2)(1:K)}_{\patch})_{(c^{(l-2)},g^{(l-1)})}\|^{\alpha+\epsilon} \Big| h_C^{(l-3)} \Big] \Big| h_C^{(l-2)} \Big] \Bigg)^{\beta }
\end{split}
\]

which is bounded by hypothesis induction.

\subsection*{Proof of L2}
By induction hypothesis, $h_C^{(l-1)}$ converges to $h^{(l-1)}$ in distribution with respect to the weak topology. Since the limit law is degenerate on $h^{(l-1)}$ (in the sense that it provides a.s.  the distribution $q^{(l-1)}$), then for every sub-sequence $(C')$ there exists a sub-sequence $(C'')$ such that $h_{C''}^{(l-1)}$ converges a.s. By the induction hypothesis, $h^{(l-1)}$ is absolutely continuous with respect to the Lebesgue measure. Since $\phi$ is almost everywhere continuous, and by L1.1) uniformly integrable with respect to $(h_C^{(l-1)})$ then we can write the following

\[\int \sum_{g^{(l)} \in [\textbf{G}^{(l)}]} |\textbf{t}^{(l)} \otimes \phi(f_{g^{(l)}})|^{\alpha} h^{(l-1)}_{C''}(\dd f_{g^{(l)} \in [\textbf{G}^{(l)}]}) \overset{a.s.}{\rightarrow} \int \sum_{g^{(l)} \in [\textbf{G}^{(l)}]} |\textbf{t}^{(l)} \otimes \phi(f_{g^{(l)}})|^{\alpha} q^{(l-1)}(\dd f_{g^{(l)} \in [\textbf{G}^{(l)}]}) \] 
Thus
\[\int \sum_{g^{(l)} \in [\textbf{G}^{(l)}]} |\textbf{t}^{(l)} \otimes \phi(f_{g^{(l)}})|^{\alpha} h^{(l-1)}_C(\dd f_{g^{(l)} \in [\textbf{G}^{(l)}]}) \overset{p}{\rightarrow} \int \sum_{g^{(l)} \in [\textbf{G}^{(l)}]} |\textbf{t}^{(l)} \otimes \phi(f_{g^{(l)}})|^{\alpha} q^{(l-1)}(\dd f_{g^{(l)} \in [\textbf{G}^{(l)}]})\] 
as $C \rightarrow +\infty$.

\subsection*{Proof of L3}
Let $\epsilon>0$ as in L1.1), $r=\frac{\alpha+\epsilon}{\alpha}$ and $q$ such that $\frac{1}{r}+\frac{1}{q}=1$. Thus, by H\"older inequality and by the fact that, being $q>1$, for every $y \geq 0$ it holds that $(1-e^{-y})^q \leq (1-e^{-y}) \leq y$, we get

\[
\begin{split}
&\int \sum_{g^{(l)} \in [\textbf{G}^{(l)}]} \|\phi(f_{g^{(l)}})\|^{\alpha} [1-\exp\{-\frac{\sigma^{\alpha}_{\omega}}{C} \sum_{g^{(l)} \in [\textbf{G}^{(l)}]}| \textbf{t}^{(l)} \otimes \phi (f_{g^{(l)}}) |^{\alpha}\}] h^{(l-1)}_C(\dd f_{g^{(l)} \in [\textbf{G}^{(l)}]})\\
&\leq \sum_{g^{(l)} \in [\textbf{G}^{(l)}]} \Bigg[ \int \|\phi(f_{g^{(l)}})\|^{\alpha r} h^{(l-1)}_C(\dd f_{g^{(l)}}) \Bigg]^{\frac{1}{r}} \times \\
& \quad \times \Bigg[ \int [1-\exp\{-\frac{\sigma^{\alpha}_{\omega}}{C} \sum_{g^{(l)} \in [\textbf{G}^{(l)}]}| \textbf{t}^{(l)} \otimes \phi (f_{g^{(l)}}) |^{\alpha}\}]^q h^{(l-1)}_C(\dd f_{g^{(l)}})
\Bigg]^{\frac{1}{q}}\\
&\leq \sum_{g^{(l)} \in [\textbf{G}^{(l)}]} \Bigg[ \int \|\phi(f_{g^{(l)}})\|^{\alpha r} h^{(l-1)}_C(\dd f_{g^{(l)}}) \Bigg]^{\frac{1}{r}} \Bigg[ \frac{\sigma^{\alpha}_{\omega}}{C} \int \sum_{g^{(l)} \in [\textbf{G}^{(l)}]} | \textbf{t}^{(l)} \otimes \phi (f_{g^{(l)}}) |^{\alpha} h^{(l-1)}_C(\dd f_{g^{(l)}})
\Bigg]^{\frac{1}{q}}\\
&\leq \sum_{g^{(l)} \in [\textbf{G}^{(l)}]} \Bigg[ \int \|\phi(f_{g^{(l)}})\|^{\alpha r} h^{(l-1)}_C(\dd f_{g^{(l)}}) \Bigg]^{\frac{1}{r}} \times \\
& \quad \times \Bigg[ \| \textbf{t}^{(l)} \|^{\alpha} \frac{\sigma^{\alpha}_{\omega}}{C} \int \sum_{g^{(l)} \in [\textbf{G}^{(l)}]} \| \phi (f_{g^{(l)}}) \|^{\alpha} h^{(l-1)}_C(\dd f_{g^{(l)}})
\Bigg]^{\frac{1}{q}}\\
&=\Big( \| \textbf{t}^{(l)} \|^{\alpha} \frac{\sigma^{\alpha}_{\omega}}{C}\Big)^{\frac{1}{q}}\sum_{g^{(l)} \in [\textbf{G}^{(l)}]} \Bigg[ \int \|\phi(f_{g^{(l)}})\|^{\alpha +\epsilon} h^{(l-1)}_C(\dd f_{g^{(l)}}) \Bigg]^{\frac{1}{r}} \times \\
& \quad \times \Bigg[ \int \sum_{g^{(l)} \in [\textbf{G}^{(l)}]} \| \phi (f_{g^{(l)}}) \|^{\alpha} h^{(l-1)}_C(\dd f_{g^{(l)}})
\Bigg]^{\frac{1}{q}}\overset{p}{\rightarrow} 0
\end{split}
\]
as $C \rightarrow \infty$ by L1) and L1.1).

\subsection*{Conclusion: L1 + L1.1 + L2 + L3}
By Lagrange theorem for $y>0$ there exists $\theta \in (0,1)$ such that $e^{-y}=1-y+y(1-e^{-y\theta})$. In our case, for $y=y_C(f_{g^{(l)}\in [\textbf{G}^{(l)}]})=\frac{\sigma^{\alpha}_{\omega}}{C} \sum_{g^{(l)} \in [\textbf{G}^{(l)}]}| \textbf{t}^{(l)} \otimes \phi (f_{g^{(l)}}) |^{\alpha}$, for any $C$ there exists a $\theta_C \in (0,1)$ such that, from (\ref{eq:conditioned}), the follow equality holds

\[
\begin{split}
&\varphi_{\big( f^{(l)(1:K)}_{(c^{(l)},:)}\big)} (\textbf{t}^{(l)})\\
=& e^{ -\sigma^{\alpha}_b|\textbf{t}^{(l)} \otimes  \1_{(\textbf{P}^{(l)} \times K)}|^{\alpha} } \E \Big[ \Big( \int e^ { - y_C\big(f_{g^{(l)}\in [\textbf{G}^{(l)}]}\big)} h^{(l-1)}_C(\dd f_{\{g^{(l)}\in [\textbf{G}^{(l)}]\}}) \Big)^{C} \Big]\\
=&e^{ -\sigma^{\alpha}_b|\textbf{t}^{(l)} \otimes  \1_{(\textbf{P}^{(l)} \times K)}|^{\alpha} } \E \Big[ \Big( 1- \int y_C\big(f_{g^{(l)}\in [\textbf{G}^{(l)}]}\big) h^{(l-1)}_C(\dd f_{\{g^{(l)}\in [\textbf{G}^{(l)}]\}})+\\
& \quad +\int y_C\big(f_{g^{(l)}\in [\textbf{G}^{(l)}]}\big)\Big[1-e^ { - \theta_C y_C\big(f_{g^{(l)}\in [\textbf{G}^{(l)}]}\big)}\Big] h^{(l-1)}_C(\dd f_{\{g^{(l)}\in [\textbf{G}^{(l)}]\}}) \Big)^{C} \Big]\\
=& \exp \Big\{ -\sigma^{\alpha}_b|\textbf{t}^{(l)} \otimes  \1_{(\textbf{P}^{(l)} \times K)}|^{\alpha} \Big\} \times \\
\quad & \times \E \Bigg[  \Bigg( 1- \frac{\sigma^{\alpha}_{\omega}}{C} \int \sum_{g^{(l)} \in [\textbf{G}^{(l)}]}| \textbf{t}^{(l)} \otimes \phi (f_{g^{(l)}}) |^{\alpha} h^{(l-1)}_C(\dd f_{\{g^{(l)}\in [\textbf{G}^{(l)}]\}}) + \\
\quad & + \frac{\sigma^{\alpha}_{\omega}}{C} \int \sum_{g^{(l)} \in [\textbf{G}^{(l)}]}| \textbf{t}^{(l)} \otimes \phi (f_{g^{(l)}}) |^{\alpha} \Big[1-\exp \Big\{ -\theta_C \frac{\sigma^{\alpha}_{\omega}}{C} \sum_{g^{(l)} \in [\textbf{G}^{(l)}]}| \textbf{t}^{(l)} \otimes \phi (f_{g^{(l)}}) |^{\alpha} \Big\}\Big] \times \\
\quad & \times h^{(l-1)}_C(\dd f_{\{g^{(l)}\in [\textbf{G}^{(l)}]\}}) \Bigg)^{C} \Bigg]
\end{split}
\]

The last integral tends to $0$ as $C \rightarrow \infty$ since by Cauchy inequality $| \textbf{t}^{(l)} \otimes \phi (f_{g^{(l)}}) |^{\alpha} \leq \|\textbf{t}^{(l)}\|^{\alpha} \|\phi (f_{g^{(l)}}) \|^{\alpha}$ we have

\[
\begin{split}
&\int \sum_{g^{(l)} \in [\textbf{G}^{(l)}]}| \textbf{t}^{(l)} \otimes \phi (f_{g^{(l)}}) |^{\alpha} \Big[1-\exp \Big\{ -\theta_C \frac{\sigma^{\alpha}_{\omega}}{C} \sum_{g^{(l)} \in [\textbf{G}^{(l)}]}| \textbf{t}^{(l)} \otimes \phi (f_{g^{(l)}}) |^{\alpha} \Big\}\Big] \times \\
\quad & \times h^{(l-1)}_C(\dd f_{\{g^{(l)}\in [\textbf{G}^{(l)}]\}})\\
&\leq \|\textbf{t}^{(l)}\|^{\alpha} \int \sum_{g^{(l)} \in [\textbf{G}^{(l)}]} \|\phi (f_{g^{(l)}}) \|^{\alpha} \Big[1-\exp \Big\{ -\theta_C \frac{\sigma^{\alpha}_{\omega}}{C} \sum_{g^{(l)} \in [\textbf{G}^{(l)}]}| \textbf{t}^{(l)} \otimes \phi (f_{g^{(l)}}) |^{\alpha} \Big\}\Big] \times \\
\quad & \times h^{(l-1)}_C(\dd f_{\{g^{(l)}\in [\textbf{G}^{(l)}]\}})\\
&\leq \|\textbf{t}^{(l)}\|^{\alpha} \int \sum_{g^{(l)} \in [\textbf{G}^{(l)}]} \|\phi (f_{g^{(l)}}) \|^{\alpha} \Big[1-\exp \Big\{ - \frac{\sigma^{\alpha}_{\omega}}{C} \sum_{g^{(l)} \in [\textbf{G}^{(l)}]}| \textbf{t}^{(l)} \otimes \phi (f_{g^{(l)}}) |^{\alpha} \Big\}\Big] \times \\
\quad & \times h^{(l-1)}_C(\dd f_{\{g^{(l)}\in [\textbf{G}^{(l)}]\}})\\
\end{split}
\]

which tends to 0 in probability by L3). Thus, by using the definition of the exponential function $e^x=\lim_{n \rightarrow \infty} (1+x/n)^n$, and by L2), we get

\[
\begin{split}
&\varphi_{\big(f^{(l)(1:K)}_{(c^{(l)},:)}\big)} (\textbf{t}^{(l)}) \\
&\rightarrow \exp \Big\{ -\sigma^{\alpha}_b|\textbf{t}^{(l)} \otimes  \1_{(\textbf{P}^{(l)} \times K)}|^{\alpha} -\sigma^{\alpha}_{\omega} \int \sum_{g^{(l)} \in [\textbf{G}^{(l)}]}| \textbf{t}^{(l)} \otimes \phi (f_{g^{(l)}}) |^{\alpha} q^{(l-1)}(\dd f_{\{g^{(l)}\in [\textbf{G}^{(l)}]\}}) \Big\}\\
&= \exp \Big\{ -\sigma^{\alpha}_b \|\1_{(\textbf{P}^{(l)} \times K)} \|^{\alpha} \Big|\textbf{t}^{(l)} \otimes \frac{ \1_{(\textbf{P}^{(l)} \times K)} }{\| \1_{(\textbf{P}^{(l)} \times K)}\|}\Big|^{\alpha} + \\
& \quad -\sigma^{\alpha}_{\omega} \int \sum_{g^{(l)} \in [\textbf{G}^{(l)}]} \| \phi (f_{g^{(l)}}) \|^{\alpha} \Big| \textbf{t}^{(l)} \otimes \frac{ \phi (f_{g^{(l)}})} {\| \phi (f_{g^{(l)}}) \|} \Big|^{\alpha} q^{(l-1)}(\dd f_{\{g^{(l)}\in [\textbf{G}^{(l)}]\}}) \Big\}\\
&= \exp \Big\{ - \int_{\mathbb{S}^{|\textbf{P}^{(l)}\times K|-1}} | \textbf{t}^{(l)} \otimes s^{(l)}|^{\alpha} \Gamma^{(l)}_{\infty} (\dd s^{(l)}) \Big\}
\end{split}
\]
where

\[
\begin{split}
    \Gamma^{(l)}_{\infty} = & \| \sigma_b \1_{(\textbf{P}^{(l)} \times K)} \|^{\alpha} \Psi^{(l)} \Big( \1_{(\textbf{P}^{(l)} \times K)} \Big)+ \\
    &+ \int \sum_{g^{(l)} \in [\textbf{G}^{(l)}]} \|\sigma_{\omega} \phi (f_{g^{(l)}})\|^{\alpha} \Psi^{(l)} \Big( \phi (f_{g^{(l)}}) \Big) q^{(l-1)}(\dd f_{\{g^{(l)}\in [\textbf{G}^{(l)}]\}})
\end{split}
\]
\end{proof}

\section*{SM A.1}\zlabel{sec:appA.1}

We prove that $f^{(1)(1:K)}_{(c^{(1)},:)} \sim \text{St}_{\textbf{P}^{(1)} \times K}(\alpha,\Gamma^{(1)})$, where 
\[
\begin{split}
    \Gamma^{(1)}=& \| \sigma_b \1_{(\textbf{P}^{(1)} \times K)} \|^{\alpha} \Psi^{(1)} ( \1_{(\textbf{P}^{(1)} \times K)} )\\
    &+ \sum_{(c^{(0)},g^{(1)}) \in [C^{(0)} \times \textbf{G}^{(1)}]} \|\sigma_{\omega} (x^{(1:K)}_{\patch})_{(c^{(0)},g^{(1)})}\|^{\alpha} \Psi^{(1)} \Big(  (x^{(1:K)}_{\patch})_{(c^{(0)},g^{(1)})}\Big)
\end{split}
\]

\begin{proof}
For $l=1$, from definition (\ref{sistem:fixedchannel}) and assumption (\ref{eq:param_dist2}) we have that, for any $\textbf{t}^{(1)}:=[t^{(1)(k)}_{p^{(1)}}]_{\{(p^{(1)},k) \in [\textbf{P}^{(1)} \times K] \}} \in \R^{\textbf{P}^{(1)} \times K}$
\[
\begin{split}
& \varphi_{(f^{(1)(1:K)}_{(c^{(1)},:)})} (\textbf{t}^{(1)})\\
& = \E \Big[ \exp \Big\{ \ii \textbf{t}^{(1)} \otimes   f^{(1)(1:K)}_{(c^{(1)},:)} \Big\} \Big] \\
& = \E \Big[ \exp \Big\{ \ii \textbf{t}^{(1)} \otimes \Big( W^{(1)}_{(c^{(1)},:,:)} \overset{(\textbf{P}^{(1)},K)}{\otimes} x^{(1:K)}_{\patch } + b^{(1)}_{c^{(1)}} \1 _{( \textbf{P}^{(1)} \times K)} \Big) \Big\} \Big] \\
& = \E \Big[ \exp \Big\{ \ii  W^{(1)}_{(c^{(1)},:,:)} \overset{(\textbf{P}^{(1)},K)}{\otimes} \Big(  \textbf{t}^{(1)} \overset{(C^{(0)}, \textbf{G}^{(1)})}{\otimes} x^{(1:K)}_{\patch }\Big) + b^{(1)}_{c^{(1)}} \textbf{t}^{(1)} \otimes \1 _{( \textbf{P}^{(1)} \times K)} \Big\} \Big] \\
& = \E \Big[\exp \Big\{ b^{(1)}_{c^{(1)}} \textbf{t}^{(1)} \otimes \1_{(\textbf{P}^{(1)} \times K)} \Big\} \Big] \times \\
& \quad \times \prod_{(c^{(0)},g^{(1)}) \in [C^{(0)} \times \textbf{G}^{(1)}]}\E \Big[ \exp \Big\{ \ii  W^{(1)}_{(c^{(1)},c^{(0)},g^{(1)})} \textbf{t}^{(1)} \otimes (x^{(1:K)}_{\patch})_{(c^{(0)},g^{(1)})}\Big) \Big\} \Big] \\
&= e^{-\sigma_b^{\alpha}|\textbf{t}^{(1)} \otimes \1_{(\textbf{P}^{(1)} \times K)}|^{\alpha}} \prod_{(c^{(0)},g^{(1)}) \in [C^{(0)} \times \textbf{G}^{(1)}]} e^{-\sigma_{\omega}^{\alpha}\Big|\textbf{t}^{(1)} \otimes (x^{(1:K)}_{\patch})_{(c^{(0)},g^{(1)})} \Big|^{\alpha}} \\
&= \exp \Big\{-\sigma_b^{\alpha}|\textbf{t}^{(1)} \otimes \1_{(\textbf{P}^{(1)} \times K)}|^{\alpha} -\sigma_{\omega}^{\alpha} \sum_{(c^{(0)},g^{(1)}) \in [C^{(0)} \times \textbf{G}^{(1)}]} \Big|\textbf{t}^{(1)} \otimes (x^{(1:K)}_{\patch})_{(c^{(0)},g^{(1)})} \Big|^{\alpha} \Big\} \\
&= \exp \Big\{-\sigma_b^{\alpha} \| \1_{(\textbf{P}^{(1)} \times K)} \|^{\alpha} \Big|\textbf{t}^{(1)} \otimes \frac{\1_{(\textbf{P}^{(1)} \times K)}}{\|\1_{(\textbf{P}^{(1)} \times K)}\|}\Big|^{\alpha} +\\
& \qquad -\sigma_{\omega}^{\alpha} \sum_{(c^{(0)},g^{(1)}) \in [C^{(0)} \times \textbf{G}^{(1)}]} \| (x^{(1:K)}_{\patch})_{(c^{(0)},g^{(1)})}\|^{\alpha} \Big|\textbf{t}^{(1)} \otimes \frac{ (x^{(1:K)}_{\patch})_{(c^{(0)},g^{(1)})}}{\| (x^{(1:K)}_{\patch})_{(c^{(0)},g^{(1)})}\|} \Big|^{\alpha} \Big\} \\
&= \exp \Big\{ - \int_{\mathbb{S}^{|\textbf{P}^{(1)}\times K|-1}} | \textbf{t}^{(1)} \otimes s^{(1)}|^{\alpha} \Gamma^{(1)} (\dd s^{(1)}) \Big\}
\end{split}
\]

where $(x^{(1:K)}_{\patch})_{(c^{(0)},g^{(1)})}=\Big[ (x^{(k)}_{\patch p^{(1)}})_{(c^{(0)},g^{(1)})} \Big]_{\{(p^{(1)},k)\in [\textbf{P}^{(1)} \times K]\}}$ and 

\[
\begin{split}
    \Gamma^{(1)}=& \| \sigma_b \1_{(\textbf{P}^{(1)} \times K)} \|^{\alpha} \Psi^{(1)} ( \1_{(\textbf{P}^{(1)} \times K)} )\\
    &+ \sum_{(c^{(0)},g^{(1)}) \in [C^{(0)} \times \textbf{G}^{(1)}]} \|\sigma_{\omega} (x^{(1:K)}_{\patch})_{(c^{(0)},g^{(1)})}\|^{\alpha} \Psi^{(1)} \Big(  (x^{(1:K)}_{\patch})_{(c^{(0)},g^{(1)})}\Big)
\end{split}
\]


\end{proof}

\section*{SM A.2}\zlabel{sec:appA.2}

We prove that for each $l=2, \dots ,L$,  $f^{(l)(1:K)}_{(c^{(l)},:)}|f^{(l-1)(1:K)}_{(1:C,:)} \sim \text{St}_{\textbf{P}^{(l)} \times K}(\alpha,\Gamma^{(l)})$, where
\[
\begin{split}
    \Gamma^{(l)}_C=& \| \sigma_b \1_{(\textbf{P}^{(l)} \times K)} \|^{\alpha} \Psi^{(l)} \Big( \1_{(\textbf{P}^{(l)} \times K)} \Big)+ \\
    &+ \frac{1}{C}\sum_{(c^{(l-1)},g^{(l)}) \in [C \times \textbf{G}^{(l)}]} \|\sigma_{\omega} \phi (f^{(l-1)(1:K)}_{\patch})_{(c^{(l-1)},g^{(l)})}\|^{\alpha} \Psi^{(l)} \Big( \phi (f^{(l-1)(1:K)}_{\patch})_{(c^{(l-1)},g^{(l)})} \Big)
\end{split}
\]

\begin{proof}
For $l \geq 2$, from definition (\ref{sistem:fixedchannel}) and assumption (\ref{eq:param_dist2}) we have that, for any $\textbf{t}^{(l)}:=[t^{(l)(k)}_{p^{(l)}}]_{\{(p^{(l)},k) \in [\textbf{P}^{(l)} \times K] \}} \in \R^{\textbf{P}^{(l)} \times K}$, it holds

\[
\begin{split}
& \varphi_{\big( f^{(l)(1:K)}_{(c^{(l)},:)}|f^{(l-1)(1:K)}_{(1:C,:)} \big)} (\textbf{t}^{(l)}) \\
& = \E \Big[ \exp \Big\{ \ii \textbf{t}^{(l)} \otimes f^{(l)(1:K)}_{(c^{(l)},:)}|f^{(l-1)(1:K)}_{(1:C,:)} \Big\} \Big] \\
& = \E \Big[ \exp \Big\{ \ii  \textbf{t}^{(l)} \otimes \Big( \frac{1}{C^{1/\alpha}}W^{(l)}_{(c^{(l)},:,:)} \overset{(\textbf{P}^{(l)},K)}{\otimes}  \phi ( f^{(l-1)(1:K)}_{\patch}) + b^{(l)}_{c^{(l)}}  \1 _{( \textbf{P}^{(l)} \times K)} \Big) \Big\} \Big] \\
& = \E \Big[ \exp \Big\{ \ii \frac{1}{C^{1/\alpha}} W^{(l)}_{(c^{(l)},:,:)} \overset{(\textbf{P}^{(l)},K)}{\otimes} \Big( \textbf{t}^{(l)} \overset{( C,\textbf{G}^{(l)})}{\otimes}  \phi ( f^{(l-1)(1:K)}_{\patch})\Big) + \Big( b^{(l)}_{c^{(l)}}  \textbf{t}^{(l)} \otimes \1 _{( \textbf{P}^{(l)} \times K)} \Big) \Big\} \Big] \\
& = \E \Big[\exp \Big\{ b^{(l)}_{c^{(l)}} \textbf{t}^{(l)} \otimes \1_{(\textbf{P}^{(l)} \times K)} \Big\} \Big] \times \\
& \quad \times \prod_{(c^{(l-1)},g^{(l)}) \in [C \times \textbf{G}^{(l)}]}\E \Big[ \exp \Big\{ \ii \frac{1}{C^{1/\alpha}}  W^{(l)}_{(c^{(l)},c^{(l-1)},g^{(l)})} \Big( \textbf{t}^{(l)} \otimes \phi(f^{(l-1)(1:K)}_{\patch})_{(c^{(l-1)},g^{(l)})}\Big) \Big\} \Big] \\
&= e^{-\sigma_b^{\alpha}|\textbf{t}^{(l)} \otimes \1_{(\textbf{P}^{(l)} \times K)}|^{\alpha}} \prod_{(c^{(l-1)},g^{(l)}) \in [C \times \textbf{G}^{(l)}]} e^{-\frac{\sigma_{\omega}^{\alpha}}{C}\Big|\textbf{t}^{(l)} \otimes \phi(f^{(l-1)(1:K)}_{\patch})_{(c^{(l-1)},g^{(l)})} \Big|^{\alpha}} \\
&= \exp \Big\{-\sigma_b^{\alpha}|\textbf{t}^{(l)} \otimes \1_{(\textbf{P}^{(l)} \times K)}|^{\alpha} - \frac{\sigma_{\omega}^{\alpha}}{C}\sum_{(c^{(l-1)},g^{(l)}) \in [C \times \textbf{G}^{(l)}]} \Big|\textbf{t}^{(l)} \otimes \phi(f^{(l-1)(1:K)}_{\patch})_{(c^{(l-1)},g^{(l)})} \Big|^{\alpha} \Big\} \\
&= \exp \Big\{-\sigma_b^{\alpha} \| \1_{(\textbf{P}^{(l)} \times K)} \|^{\alpha} \Big|\textbf{t}^{(l)} \otimes \frac{\1_{(\textbf{P}^{(l)} \times K)}}{\|\1_{(\textbf{P}^{(l)} \times K)}\|}\Big|^{\alpha} +\\
& \qquad - \frac{\sigma_{\omega}^{\alpha}}{C} \sum_{(c^{(l-1)},g^{(l)}) \in [C \times \textbf{G}^{(l)}]} \|\phi (f^{(l-1)(1:K)}_{\patch})_{(c^{(l-1)},g^{(l)})}\|^{\alpha} \Big|\textbf{t}^{(l)} \otimes \frac{ (f^{(l-1)(1:K)}_{\patch})_{(c^{(l-1)},g^{(l)})}}{\| \phi(f^{(l-1)(1:K)}_{\patch})_{(c^{(l-1)},g^{(l)})}\|} \Big|^{\alpha} \Big\} \\
&= \exp \Big\{ - \int_{\mathbb{S}^{|\textbf{P}^{(l)}\times K|-1}} | \textbf{t}^{(l)} \otimes s^{(l)}|^{\alpha} \Gamma^{(l)}_C (\dd s^{(l)}) \Big\}
\end{split}
\]
where

\[
\begin{split}
    \Gamma^{(l)}_C=& \| \sigma_b \1_{(\textbf{P}^{(l)} \times K)} \|^{\alpha} \Psi^{(l)} \Big( \1_{(\textbf{P}^{(l)} \times K)} \Big)+ \\
    &+ \frac{1}{C}\sum_{(c^{(l-1)},g^{(l)}) \in [C \times \textbf{G}^{(l)}]} \|\sigma_{\omega} \phi (f^{(l-1)(1:K)}_{\patch})_{(c^{(l-1)},g^{(l)})}\|^{\alpha} \Psi^{(l)} \Big( \phi (f^{(l-1)(1:K)}_{\patch})_{(c^{(l-1)},g^{(l)})} \Big)
\end{split}
\]

with $ (f^{(l-1)(1:K)}_{\patch})_{(c^{(l-1)},g^{(l)})}=\Big[ (f^{(l-1)(k)}_{\patch p^{(l)}})_{(c^{(l-1)},g^{(l)})} \Big]_{\{(p^{(l)},k)\in [\textbf{P}^{(l)} \times K]\}}$.
\end{proof}

\section*{SM B}\zlabel{sec:appB}

We prove that for each $l \in [L]$, $f^{(l)(1:K)} = f^{(l)}(x^{(1:K)},C) \overset{d}{\rightarrow} \bigotimes_{c^{(l)}=1}^{\infty} \text{St}_{\textbf{P}^{(l)} \times K}(\alpha,\Gamma^{(l)}_{\infty}) $ as $C \rightarrow \infty$.

The symbol $\bigotimes$ here denotes the product measure. The proof follows by the Cram\'er-Wold theorem for finite-dimensional projection of $f^{(l)(1:K)}=f^{(l)}(x^{(1:K)},C)$ for which it is sufficient to prove the large $C$ asymptotic behavior of any linear combination of the $f^{(l)(1:K)}_{(c^{(l)},:)}$'s for $c^{(l)} \in \mathcal{L} \subset \N$. See, e.g. \cite{billingsley1999convergence} for details. 

\begin{proof}
Following the notation of \cite{matthews2018gaussianB}, consider a finite linear combination of the function values without the bias, i.e. fix $z=(z_{c^{(l)}})_{\{c^{(l)} \in \mathcal{L}\}}$ and define
\[
\mathcal T^{(l)}(\mathcal L,z,x^{(1:K)},C^{(l-1)})= \sum_{c^{(l)} \in \mathcal L} z_{c^{(l)}}[f^{(l)(1:K)}_{(c^{(l)},:)}-b_{c^{(l)}}^{(l)} \1_{(\textbf{P}^{(l)}\times K)}].
\]

The case $l=1$ is easy since it does not depend on $C$, indeed we get
\[
\begin{split}
\mathcal T^{(1)}(\mathcal L,z,x^{(1:K)},C^{(0)}) &= \sum_{c^{(1)} \in \mathcal L} z_{c^{(1)}} W^{(1)}_{(c^{(1)},:,:)} \overset{(\textbf{P}^{(1)},K)}{\otimes}  x^{(1:K)}_{\patch}
\end{split}
\]
and, following the same steps as in Theorem \ref{theorem1}, for any $\textbf{t}^{(1)}:=[t^{(1)(k)}_{p^{(1)}}]_{\{(p^{(1)},k) \in [\textbf{P}^{(1)} \times K] \}} \in \R^{\textbf{P}^{(1)} \times K}$, called $\| z \|^{\alpha}= \sum_{c^{(1)} \in \mathcal{L}} |z_{c^{(1)}}|^{\alpha}$, we get

\[
\begin{split}
& \varphi_{\big( T^{(1)}(\mathcal L,z,x^{(1:K)},C^{(0)}) \big)} (\textbf{t}^{(1)}) \\
& = \E \Big[ \exp \Big\{ \ii \textbf{t}^{(1)} \otimes T^{(1)}(\mathcal L,z,x^{(1:K)},C^{(0)}) \Big\} \Big] \\
& = \E \Big[ \exp \Big\{ \ii  \textbf{t}^{(1)} \otimes \Big( \sum_{c^{(1)} \in \mathcal L} z_{c^{(1)} } W^{(1)}_{(c^{(1)},:,:)} \overset{(\textbf{P}^{(1)},K)}{\otimes} x^{(1:K)}_{\patch} \Big) \Big\} \Big] \\
& = \E \Big[ \exp \Big\{ \ii \sum_{c^{(1)} \in \mathcal L} z_{c^{(1)}}  W^{(1)}_{(c^{(1)},:,:)} \overset{(\textbf{P}^{(1)},K)}{\otimes} \Big( \textbf{t}^{(1)} \overset{(C^{(0)},\textbf{G}^{(1)})}{\otimes}  x^{(1:K)}_{\patch} \Big) \Big\} \Big] \\
& = \prod_{(c^{(1)},c^{(0)},g^{(1)}) \in \mathcal L\times [C^{(0)} \times \textbf{G}^{(1)}]}\E \Big[ \exp \Big\{ \ii z_{c^{(1)}} W^{(1)}_{(c^{(1)},c^{(0)},g^{(1)})} \textbf{t}^{(1)} \otimes (x^{(1:K)}_{\patch})_{(c^{(0)},g^{(1)})}\Big) \Big\} \Big] \\
&= \prod_{(c^{(1)},c^{(0)},g^{(1)}) \in \mathcal L\times [C^{(0)} \times \textbf{G}^{(1)}]} \exp \Big\{-(|z_{c^{(1)}}|\sigma_{\omega})^{\alpha}\Big|\textbf{t}^{(1)} \otimes (x^{(1:K)}_{\patch})_{(c^{(0)},g^{(1)})} \Big|^{\alpha} \Big\} \\
&= \exp \Big\{ - \sigma_{\omega}^{\alpha} \| z \|^{\alpha} \sum_{(c^{(0)},g^{(1)}) \in [C^{(0)} \times \textbf{G}^{(1)}]} \Big|\textbf{t}^{(1)} \otimes (x^{(1:K)}_{\patch})_{(c^{(0)},g^{(1)})} \Big|^{\alpha} \Big\} \\
&= \exp \Big\{- \sigma_{\omega}^{\alpha} \|z\|^{\alpha} \sum_{(c^{(0)},g^{(1)}) \in [C^{(0)} \times \textbf{G}^{(1)}]} \| (x^{(1:K)}_{\patch})_{(c^{(0)},g^{(1)})}\|^{\alpha} \Big|\textbf{t}^{(1)} \otimes \frac{ (x^{(1:K)}_{\patch})_{(c^{(0)},g^{(1)})}}{\| (x^{(1:K)}_{\patch})_{(c^{(0)},g^{(1)})}\|} \Big|^{\alpha} \Big\} \\
&= \exp \Big\{ - \int_{\mathbb{S}^{|\textbf{P}^{(1)}\times K|-1}} | \textbf{t}^{(1)} \otimes s^{(1)}|^{\alpha} \Delta^{(1)}_C (\dd s^{(1)}) \Big\}
\end{split}
\]

where $\Delta^{(1)}_C$ coincides with $\Gamma^{(1)}$ just replacing $\sigma_{b} \leftarrow 0$ and $\sigma_{\omega} \leftarrow \sigma_{\omega}\|z\|$. Thus since the characteristic function does not depend on $C$ we get (as $C \rightarrow \infty$)

\[
\mathcal T^{(1)}(\mathcal L,z,x^{(1:K)},C^{(0)}) \overset{d}{\rightarrow} \text{St}_{\textbf{P}^{(1)} \times K}(\alpha,\Delta^{(1)}_{\infty})
\]
 where 
 \[
 \Delta^{(1)}=\Delta^{(1)}_{\infty} = \| z \|^{\alpha} \sum_{(c^{(0)},g^{(1)}) \in [C^{(0)} \times \textbf{G}^{(1)}]} \|\sigma_{\omega} (x^{(1:K)}_{\patch})_{(c^{(0)},g^{(1)})}\|^{\alpha} \Psi^{(1)} \Big(  (x^{(1:K)}_{\patch})_{(c^{(0)},g^{(1)})}\Big)
    \]

For $l=2,\dots, L$,
\[
\begin{split}
\mathcal T^{(l)}(\mathcal L,z,x^{(1:K)},C) &= \sum_{c^{(l)} \in \mathcal L} \frac{z_{c^{(l)}} }{C^{1/\alpha}} W^{(l)}_{(c^{(l)},:,:)} \overset{(\textbf{P}^{(l)},K)}{\otimes}  \phi ( f^{(l-1)(1:K)}_{\patch})
\end{split}
\]
and, following the same steps as in Theorem \ref{theorem2}, for any $\textbf{t}^{(l)}:=[t^{(l)(k)}_{p^{(l)}}]_{\{(p^{(l)},k) \in [\textbf{P}^{(l)} \times K] \}} \in \R^{\textbf{P}^{(l)} \times K}$, called $\| z \|^{\alpha}= \sum_{c^{(l)} \in \mathcal{L}} |z_{c^{(l)}}|^{\alpha}$, we get

\[
\begin{split}
& \varphi_{\big( T^{(l)}(\mathcal L,z,x^{(1:K)},C)|f^{(l-1)(1:K)}_{(1:C,:)} \big)} (\textbf{t}^{(l)}) \\
& = \E \Big[ \exp \Big\{ \ii \textbf{t}^{(l)} \otimes T^{(l)}(\mathcal L,z,x^{(1:K)},C)|f^{(l-1)(1:K)}_{(1:C,:)} \Big\} \Big] \\
& = \E \Big[ \exp \Big\{ \ii  \textbf{t}^{(l)} \otimes \Big( \sum_{c^{(l)} \in \mathcal L} \frac{z_{c^{(l)}} }{C^{1/\alpha}} W^{(l)}_{(c^{(l)},:,:)} \overset{(\textbf{P}^{(l)},K)}{\otimes}  \phi ( f^{(l-1)(1:K)}_{\patch}) \Big) \Big\} \Big] \\
& = \E \Big[ \exp \Big\{ \ii \sum_{c^{(l)} \in \mathcal L} \frac{z_{c^{(l)}} }{C^{1/\alpha}} W^{(l)}_{(c^{(l)},:,:)} \overset{(\textbf{P}^{(l)},K)}{\otimes} \Big( \textbf{t}^{(l)} \overset{(C,\textbf{G}^{(l)})}{\otimes}  \phi ( f^{(l-1)(1:K)}_{\patch}) \Big) \Big\} \Big] \\
& = \prod_{(c^{(l)},c^{(l-1)},g^{(l)}) \in \mathcal L\times [C^ \times \textbf{G}^{(l)}]}\E \Big[ \exp \Big\{ \ii \frac{z_{c^{(l)}}}{C^{1/\alpha}}  W^{(l)}_{(c^{(l)},c^{(l-1)},g^{(l)})} \textbf{t}^{(l)} \otimes \phi(f^{(l-1)[1:K]}_{\patch})_{(c^{(l-1)},g^{(l)})}\Big) \Big\} \Big] \\
&= \prod_{(c^{(l)},c^{(l-1)},g^{(l)}) \in \mathcal L\times [C \times \textbf{G}^{(l)}]} \exp \Big\{-\frac{(|z_{c^{(l)}}|\sigma_{\omega})^{\alpha}}{C}\Big|\textbf{t}^{(l)} \otimes \phi(f^{(l-1)(1:K)}_{\patch})_{(c^{(l-1)},g^{(l)})} \Big|^{\alpha} \Big\} \\
&= \exp \Big\{ - \frac{\sigma_{\omega}^{\alpha}}{C} \| z \|^{\alpha} \sum_{(c^{(l-1)},g^{(l)}) \in [C \times \textbf{G}^{(l)}]} \Big|\textbf{t}^{(l)} \otimes \phi(f^{(l-1)(1:K)}_{\patch})_{(c^{(l-1)},g^{(l)})} \Big|^{\alpha} \Big\} \\
&= \exp \Big\{- \frac{\sigma_{\omega}^{\alpha}}{C} \|z\|^{\alpha} \sum_{(c^{(l-1)},g^{(l)}) \in [C \times \textbf{G}^{(l)}]} \|\phi (f^{(l-1)(1:K)}_{\patch})_{(c^{(l-1)},g^{(l)})}\|^{\alpha} \times \\
& \quad \times \Big|\textbf{t}^{(l)} \otimes \frac{ (f^{(l-1)(1:K)}_{\patch})_{(c^{(l-1)},g^{(l)})}}{\| \phi(f^{(l-1)(1:K)}_{\patch})_{(c^{(l-1)},g^{(l)})}\|} \Big|^{\alpha} \Big\} \\
&= \exp \Big\{ - \int_{\mathbb{S}^{|\textbf{P}^{(l)}\times K|-1}} | \textbf{t}^{(l)} \otimes s^{(l)}|^{\alpha} \Delta^{(l)}_C (\dd s^{(l)}) \Big\}
\end{split}
\]

where $\Delta^{(l)}_C$ coincides with $\Gamma^{(l)}_C$ just replacing $\sigma_{b} \leftarrow 0$ and $\sigma_{\omega} \leftarrow \sigma_{\omega}\|z\|$. Now, proceeding as in Theorem \ref{teorem3}, we get the weak limit as $C \rightarrow +\infty$, i.e.

\[
\mathcal T^{(l)}(\mathcal L,z,x^{(1:K)},C) \overset{d}{\rightarrow} \text{St}_{\textbf{P}^{(l)} \times K}(\alpha,\Delta^{(l)}_{\infty})
\]
 where
 \[
\begin{split}
    \Delta^{(l)}_{\infty} = \| z \|^{\alpha} \int \sum_{g^{(l)} \in [\textbf{G}^{(l)}]} \|\sigma_{\omega} \phi (f_{g^{(l)}})\|^{\alpha} \Psi^{(l)} \Big( \phi (f_{g^{(l)}}) \Big) q^{(l-1)}(\dd f_{\{g^{(l)}\in [\textbf{G}^{(l)}]\}})
\end{split}
\]
This completes the proof.
\end{proof}

\section*{SM C}\zlabel{sec:appC}
Focus the attention on the last layer $L$. We found that $f^{(L)(1:K)} \overset{d}{\rightarrow} f^{(L)(1:K)}_{\infty}$, i.e. a convergence of a sequence of $\R^{\infty \times \textbf{P}^{(L)} \times K}$-valued random variables. To gather information on the positions $\textbf{P}^{(L)}$ we consider a linear combination with respect to $\textbf{P}^{(L)}$, i.e. we project the $\infty \times \textbf{P}^{(L)} \times K$ dimensional vector $f^{(L)(1:K)}=f^{(L)}(x^{(1:K)},C)$ into a $\infty \times K$ dimensional vectors, and we take the limit as $C \rightarrow \infty$. More precisely, for $l \in [L]$, fix $\textbf{u} \in \R^{\textbf{P}^{(l)}}$ such that $\textbf{u} \otimes \1_{(\textbf{P}^{(l)})}=1$ and define the transformation $T^{(l)}_{\textbf{u}}: \R^{\infty \times \textbf{P}^{(l)} \times K} \rightarrow \R^{ \infty \times K}$, $(a,b,c) \mapsto (a, \textbf{u} \otimes b, c)$ (in other words $T^{(l)}_{\textbf{u}} \equiv \textbf{u} \underset{\textbf{P}^{(l)}}{\otimes}$). We want to establish the convergence of $T^{(l)}_{\textbf{u}}(f^{(l)(1:K)})$ as $C \rightarrow \infty$.

For $l = 1$ we get

\begin{align*}
T^{(1)}_{\textbf{u}}(f^{(1)(1:K)})&=T^{(1)}_{\textbf{u}}\Big( W^{(1)} \overset{(\textbf{P}^{(1)},K)}{\boxdot} x^{(1:K)}_{\patch}+b^{(1)} \triangle \1 _{( \textbf{P}^{(1)} \times K)} \Big) \\
&= W^{(1)} \overset{(K)}{\boxdot} (\textbf{u} \underset{\textbf{P}^{(1)}}{\otimes} x^{(1:K)}_{\patch})+b^{(1)} \triangle \1 _{(K)}
\end{align*}

and for $l > 1$,
\[
\begin{split}
T^{(l)}_{\textbf{u}}(f^{(l)(1:K)})
&=T^{(l)}_{\textbf{u}} \Big (\frac{1}{C^{1/\alpha}}W^{(l)} \overset{(\textbf{P}^{(l)},K)}{\boxdot}\phi(f^{(l-1)(1:K)}_{\patch})+b^{(l)} \triangle \1 _{( \textbf{P}^{(l)}\times K)} \Big) \\
& = \frac{1}{C^{1/\alpha}}W^{(l)} \overset{(K)}{\boxdot}\big( \textbf{u} \underset{\textbf{P}^{(l)}}{\otimes} \phi(f^{(l-1)(1:K)}_{\patch}) \big)+b^{(l)} \triangle \1 _{( K)}
\end{split}
\]
Following the same steps of Theorem \ref{thm_Cramer} we get that 
\[
T^{(l)}_{\textbf{u}}(f^{(l)(1:K)}) \overset{d}{\rightarrow} \bigotimes_{c^{(l)}=1}^{\infty} \text{St}_{ K}(\alpha,\Gamma^{(l)}_{\infty}(\textbf{u}))
\]
where
\begin{align*}
    \Gamma^{(l)}_{\infty}(\textbf{u}) &=  \| \sigma_b \1_{(K)} \|^{\alpha} \Psi^{(l)} \big( \1_{(K)} \big)+ \int \sum_{g^{(l)} \in [\textbf{G}^{(l)}]} \|\sigma_{\omega} \textbf{u} \underset{\textbf{P}^{(l)}}{\otimes}\phi (f_{g^{(l)}})\|^{\alpha} \times \\
    & \quad \times \mathcal{D}^{(l)} \Big(\textbf{u} \underset{\textbf{P}^{(l)}}{\otimes}\phi (f_{g^{(l)}}) \Big) q^{(l-1)}(\dd f_{\{g^{(l)}\in [\textbf{G}^{(l)}]\}})
\end{align*}
where $\mathcal{D}^{(l)}:\R^{K} \to \R$
\[
\mathcal{D}^{(l)} (z) := 
\begin{cases}
 \frac{1}{2}\delta \Big(\frac{z}{\| z \|} \Big) +\frac{1}{2}\delta \Big(-\frac{z}{\| z \|} \Big) & 0 \neq z \in \R^{K}
 \\
 0 & 0 = z \in \R^{K}
\end{cases}
\]
$f_{g^{(l)}} \in \R^{\textbf{P}^{(l)} \times K}$ and $q^{(l)} = \text{St}_{\textbf{P}^{(l)} \times K}(\alpha,\Gamma^{(l)}_{\infty})$ for $l\in[L]$, being $\Gamma^{(l)}_{\infty}$ defined in theorem \ref{teorem3}.

\end{document}